\documentclass[10pt,twocolumn,letterpaper]{article}

\usepackage{cvpr}
\usepackage{times}
\usepackage{epsfig}
\usepackage{graphicx}
\usepackage{amsmath}
\usepackage{amssymb}
\usepackage{color,algorithm,algpseudocode,booktabs,bm,relsize,enumitem,multirow,amsthm,subfigure}
\usepackage{mathtools}
\usepackage{arydshln}
\usepackage{sidecap}
\mathtoolsset{showonlyrefs}

\DeclareMathOperator{\sign}{sign}
\DeclareMathOperator*{\argmax}{arg\,max}

\newtheorem{proposition}{Proposition}

\usepackage[pagebackref=true,breaklinks=true,colorlinks,bookmarks=false]{hyperref}

\cvprfinalcopy 

\def\cvprPaperID{8521} 

\begin{document}

\title{Deep Residual Flow for Out of Distribution Detection}

\author{Ev Zisselman\\
Department of Electrical Engineering\\
Technion\\
{\tt\small ev\textunderscore zis@campus.technion.ac.il}
\and
Aviv Tamar\\
Department of Electrical Engineering\\
Technion\\
{\tt\small avivt@technion.ac.il}
}

\maketitle

\begin{abstract}
The effective application of neural networks in the real-world relies on proficiently detecting out-of-distribution examples. Contemporary methods seek to model the distribution of feature activations in the training data for adequately distinguishing abnormalities, and the state-of-the-art method uses Gaussian distribution models. In this work, we present a novel approach that improves upon the state-of-the-art by leveraging an expressive density model based on normalizing flows. 
We introduce the \emph{residual flow}, a novel flow architecture that learns the residual distribution from a base Gaussian distribution. Our model is general, and can be applied to any data that is approximately Gaussian. For out of distribution detection in image datasets, our approach provides a principled improvement over the state-of-the-art.
Specifically, we demonstrate the effectiveness
of our method in ResNet and DenseNet architectures trained on various image datasets. For example, on a ResNet trained on CIFAR-100 and evaluated on detection of out-of-distribution samples from the ImageNet dataset, holding the true positive rate (TPR) at $95\%$, we improve the true negative rate (TNR) from $56.7\%$ (current state-of-the-art) to $77.5\%$ (ours). 
\end{abstract}

\section{Introduction}
Deep neural networks (DNNs) are powerful models that achieve high performance in various tasks in computer vision \cite{krizhevsky2012imagenet}, speech and audio recognition \cite{hinton2012deep}, and language processing \cite{Cho_learningphrase}. Leading DNN architectures are known to generalize well and achieve impressive performance when evaluated on samples drawn from the distribution observed at the training phase \cite{Cho_learningphrase, he2016deep, huang2017densely, krizhevsky2012imagenet, simonyan2014very}. However, DNNs tend to behave unexpectedly when encountering input taken from an unfamiliar distribution. In such instances, an out-of-distribution (OOD) input causes the majority of models to mispredict, often with high confidence \cite{Goodfellow2014ExplainingAH, lee2018simple, moosavi2017universal, nguyen2015deep, Szegedy2013IntriguingPO}. This behaviour poses a severe concern about the reliability of predictions made by DNNs and hinders their applicability to real-world scenarios \cite{amodei2016concrete}.

Contemporary work aimed at predicting classification uncertainty adopt an approach of constructing a confidence score based on characteristics of the feature space of trained neural networks. In \cite{hendrycks2016baseline}, Hendrycks and Gimpel propose a baseline method, which taps into features of the penultimate layer and uses the soft-max score as the confidence score. Their method is further improved by Liang et al. \cite{liang2017principled}, who incorporate the soft-max score with temperature scaling, alongside input pre-processing that emphasizes the score difference between in- and out-of-distribution samples.
The current state-of-the-art is the method of Lee et al. \cite{lee2018simple}, which models the feature distribution in different layers of a trained network by a Gaussian distribution under the LDA assumption (i.e., different mean but same covariance for different classes), and forms a confidence score for each layer based on the posterior distribution of the LDA model, averaged over different layers.
Lee's method shows superior performance compared with previous methods; in some cases surpassing by a large margin~ \cite{lee2018simple}. 

Building on the observation that a Gaussian model of network activations is an effective confidence measure, in this work we ask: can we improve OOD detection performance by using more expressive distributions of network activations? In particular, there is no reason to expect that features in mid-layers of the network follow an exact Gaussian distribution, and we expect that a more expressive model should capture their distribution more accurately.

We present a new approach for OOD detection and propose a more expressive density function, based on deep normalizing flow, for modeling the distribution of the feature space of trained neural networks. As a prelude, we posit that training a linear flow on the feature space of neural networks is equal to fitting a Gaussian distribution, as proposed in \cite{lee2018simple}. Then, we leverage this property to propose a novel flow architecture that adds a non-linear residual to the linear flow to produce a more expressive mapping.
The residual flow model is of independent interest, and should be effective for any data that is approximately Gaussian distributed. For out-of-distribution detection in image classification, modeling the network activations as a residual from Gaussian distribution allows us a principled improvement over the state-of-the-art, and in some cases yields superior performance by a large margin. Furthermore, the proposed residual flow model enables class-conditional density learning that improves performance, even in cases of limited training examples from each class (as in CIFAR100). Lastly, to make in- and out-of-distribution samples more separable, we extend the input preprocessing ideas of ~\cite{liang2017principled,lee2018simple} to our flow-based model, and perturb test samples to increase their likelihood under our model. We show that this perturbation can increase the contrast between in- and out-of-distribution samples, leading to further performance improvement.

We demonstrate the effectiveness of our method using trained convolutional neural networks such as DenseNet \cite{huang2017densely} and ResNet \cite{he2016deep}, trained on various datasets, and tested on various out-of-distribution examples. Our method outperforms the state-of-the-art method \cite{lee2018simple} for detecting out-of-distribution samples in all tested cases. For example, for a ResNet trained on CIFAR-100, we improve the true negative rate (TNR) of detecting samples from the LSUN dataset at a true positive rate (TPR) of 95$\%$ (i.e. 95$\%$ of the CIFAR-100 test images were correctly classified) from 38.4$\%$ \cite{lee2018simple} to 70.4$\%$ (ours), with all hyper-parameters tuned strictly from the training dataset. Our results demonstrate that the feature space of neural networks does not necessarily conform with a Gaussian distribution, and a more accurate model can significantly improve confidence estimates. 


\section{Background}
\label{sec:background}
We present preliminaries on normalizing flows and OOD detection.
\subsection{Normalizing Flows for Density Estimation}
Normalizing flows are an effective model for high-dimensional data distributions, originally studied in classical statistics~\cite{tabak2013family, tabak2010density}, and recently popularized in the deep learning community (e.g., NICE \cite{Dinh2014NICENI}, RealNVP \cite{Dinh2016DensityEU}, and GLOW \cite{kingma2018glow}). 
Let $x \in X$ denote data sampled from an unknown distribution $x \sim p_X(x)$. The main idea in normalizing flows is to represent $p_X(x)$ as a transformation of a Gaussian distribution $z \sim p_Z(z) = \mathcal{N}(0,I)$, i.e. $x = g(z)$.
Moreover, we assume the mapping to be bijective $x = g(z) = f^{-1}(z)$. As such, the data log-likelihood is given by the change of variable formula:
\begin{align}\label{eq:2}
\log\left(p_X(x)\right) =& \log\left(p_Z\left(f(x)\right)\right)\\
&+\log\left( \left|\det\left(\frac{\partial f(x)}{\partial x^T}\right)\right|\right),
\end{align}
where $\frac{\partial f(x)}{\partial x^T}$ is the Jacobian of the map $f(x)$ at $x$. The functions $f,g$ can be learned by maximum likelihood, where the bijectivity assumption allows to train expressive mappings, such as deep neural networks by backpropagation. Further, given a sample $x$, its likelihood can be inferred from \eqref{eq:2}.

To achieve a tractable, yet flexible Jacobian for the map $f(x)$, the authors of NICE \cite{Dinh2014NICENI} and RealNVP \cite{Dinh2016DensityEU} proposed to stack a sequence of simple bijective transformations, such that their Jacobian is a triangular matrix. This way, its log-determinant is simply determined by the sum of its diagonal elements. In NICE \cite{Dinh2014NICENI}, the authors proposed the \textit{additive coupling layer} for each transformation. This was further improved in RealNVP \cite{Dinh2016DensityEU} which proposed the \textit{affine coupling layer}. In each affine coupling transformation, the input vector $x\in \mathbb{R}^d$ is split into upper and lower halves,  $x_{1},x_{2} \in \mathbb{R}^{d/2}$. These are plugged into the following transformation, referred to as a single flow-block  $f_i$:
\begin{align}\label{eq:3}
z_1 = x_1,~~~ z_2 = x_2 \circ \exp(s_i(x_1)) + t_i(x_1),
\end{align}
where $\circ$ denotes element-wise multiplication, and $s_i$ and $t_i$ are non-linear mappings (e.g., deep neural networks) that need not be invertible.
Given the output $z_1$ and $z_2$, this affine transformation is trivially invertible by:
\begin{align}\label{eq:4}
x_1 = z_1,~~~ x_2 = (z_2 - t_i(z_1)) \circ \exp(-s_i(z_1)).
\end{align}
Let $r$ denote a switch-permutation, which permutes the order of $x_1$ and $x_2$. A RealNVP flow comprises $k$ reversible flow-blocks interleaved with switch-permutations,\footnote{The RealNVP paper \cite{Dinh2016DensityEU} also considered other types of permutations, such as checkerboard masks for 2-dimensional image input. Here, we focus on 1-dimensional data, and only consider the switch-permutation, which was first proposed in \cite{Dinh2014NICENI}. } 
\begin{equation}
    f_{RealNVP} = f_k\cdot r \dots f_2 \cdot r \cdot f_1.
\end{equation}
According to the chain rule, the log-determinant of the Jacobian of the whole transformation $f$ is computed by summing the log-determinant of the Jacobian of each $f_i$, making the likelihood computation \eqref{eq:2} tractable.

In GLOW~\cite{kingma2018glow}, additional permutations between flow-blocks are added, to reduce the structural constraint of separating the input into two halves:
\begin{equation}
    f_{GLOW} = f_k\cdot p_{k-1} \dots f_3\cdot p_2 \cdot f_2\cdot p_1 \cdot f_1,
\end{equation}
where $p_i$ are either fixed (random) or learned permutation matrices. Since permutations are easily inverted and $|\det(p_i)|\!=\!1$, the log-likelihood \eqref{eq:2} remains tractable.

\subsection{Out of Distribution detection}
Consider a deep neural network classifier trained in the standard supervised learning setting (via labeled data). The OOD detection problem seeks to assign a confidence score to the classifier predictions, such that classification of OOD data would be given a lower score than in-distribution data. 
Liang et al.~\cite{liang2017principled} applied temperature-scaling to the network's soft-max output as the confidence score. Let $\sigma_i(x)$ denote the network's logit output for class $i$ and input $x$. Then the temperature-scaled (TS) score is:
$$S_{TS}(x; T) = \max_i \left( \frac{\exp (\sigma_i(x)/T)}{\sum^N_{j=1} \exp (\sigma_j (x)/T)} \right),$$
where $T$ is the temperature. In addition, Liang et al.~\cite{liang2017principled} proposed to pre-process the input $x$ by modifying it in a direction that increases the soft-max score:
$$
\Tilde{x}_{TS}(x) = x - \epsilon \cdot \sign\left(- \nabla_x \log S_{TS}(x; T)\right),
$$
where the intuition is that in-distribution samples would be more susceptible to an informative pre-processing, leading to better discrimination between in- and out-of-distribution samples. The final method, termed ODIN is given by: 
$$
S_{ODIN}(x; T) = S_{TS}(\Tilde{x}_{TS}(x); T).
$$

Lee et al.~\cite{lee2018simple} improve on the ODIN method by considering different layers of the network, and measuring the Mahalanobis distance from the average network activations. For some network layer $l$ and class label $c$, let $\phi_l(x)$ denote the feature activations at layer $l$ for input $x$.\footnote{For a convolutional neural network, \cite{lee2018simple} propose to take the average activation across the spatial dimensions for each channel. In this work we adopt this approach, but our method can be applied without change to the actual feature activations.} Let $\hat{\mu}_{l,c}$ denote the empirical mean of feature activations for training data from class $c$, and let $\hat{\Sigma}_l$ denote the empirical covariance matrix of feature activations, calculated across all classes.
Given a test example $x$, Lee et al.~\cite{lee2018simple} calculate the score as the weighted Mahalanobis distance:
$$
S_M(x)\!=\!\!\sum_l w_l \!\cdot\! \max_c \{\!-\! \left(\phi_l(x) \!-\! \hat{\mu}_{l,c}\right)^T \!\hat{\Sigma}_l^{-1} \!\left(\phi_l (x) \!-\! \hat{\mu}_{l,c} \right)\},
$$
where $w_l$ are weights. Using the Mahalanobis distance as a score is equivalent to modeling the feature space of every layer as a $C$ class-conditional Gaussian distribution with a tied covariance $\hat{\Sigma}$, i.e., $P(\phi_l(x)|y\!=\!\!c)\!=\!\!\mathcal{N}(\phi_l(x)|\hat{\mu}_{l,c}, \hat{\Sigma})$, and measuring the score as the likelihood of the features (under the most likely class, and averaging over all layers). 

Lee et al.~\cite{lee2018simple} motivate the Mahalanobis score from a connection between the softmax output of the final layer and a generative classifier
with a class-conditional Gaussian distribution model with tied covariance. This generative model is a special case of Gaussian discriminant analysis (GDA), also known as linear discriminant analysis (LDA). 

Lee et al.~\cite{lee2018simple} also propose a pre-processing method similar to ODIN, where\\
$\Tilde{x}_{M}(x) \!=\! x - \epsilon \cdot \sign\left(\nabla_x \left(\phi_l(x) \!-\! \hat{\mu}_{l,\hat{c}}\right)^T \!\hat{\Sigma}_l^{-1} \!\left(\phi_l (x) \!-\! \hat{\mu}_{l,\hat{c}} \right)\right)$.

\section{Residual Flow for OOD Detection}\label{sec:residual_flow}
Our aim is to detect out of distribution (OOD) examples, equipped with an already trained neural network classifier at our disposal. This is achieved by learning the distribution of the feature space of various layers of the network, given valid, in-distribution inputs that were observed during the training phase. Motivated by the empirical success of the Gaussian distribution model of Lee at al.~\cite{lee2018simple}, 
in this section we propose a normalizing flow architecture that allows for a principled extension of the Gaussian model to non-Gaussian distributions. We hypothesize that the activations of general neural network layers do not necessarily follow a Gaussian distribution, and thus a more expressive model should allow for better OOD detection performance. Our model is composed of a linear component, which we show is equivalent to a Gaussian model, and a non-linear residual component, which allows to fit more expressive distributions using deep neural network flow architecture. 

\subsection{Linear Flow Model}\label{subsec:linear_flow_model}
We start by establishing a simple relation between the maximum-likehood estimate of a Gaussian model (as in GDA) and linear flow. The next proposition shows that for a linear flow model, the maximum likelihood parameters are equivalent to the empirical mean and covariance of the data.
\begin{proposition}\label{prop:lin}
Let $X = \{x_1,x_2, ..., x_N\}$ be a dataset of vectors in $\mathbb{R}^d$, i.e $ \forall i:~x_i\in\mathbb{R}^d$. Consider a linear normalizing flow, i.e $X = AZ+b$, where $Z\sim \mathcal{N}(0,I)$, $A\in \mathbb{R}^{d\times d}$ and $b\in \mathbb{R}^{d}$. Let $p_{A,b}(x_i)$ denote the probability of $x_i$ under this flow model. The parameters $A,b$ that maximize the likelihood of the dataset $X$ under this model satisfy:
$
b=\frac{1}{N}\sum_{i=1}^{N} x_i = \hat{\mu},
$ the empirical mean and 
$
AA^T =  \frac{1}{N}\sum_{i=1}^{N} (x_i-\hat{\mu})(x_i-\hat{\mu})^T= \hat{\Sigma},
$ the empirical covariance of the data $X$.
\end{proposition}
\begin{proof}
Since $X$ is a linear transformation of $Z \sim \mathcal{N}(0,I)$, the probability of $X$ under this model is given by:
\begin{align} \label{eq:lin_probability}
p_{A,b}(x_i)\sim \mathcal{N}(b,AA^T). 
\end{align}
On the other hand, the maximum likelihood (ML) estimators $\Tilde{\mu},\Tilde{\Sigma}$ for $X$ under Gaussian distribution assumption are known to be the empirical mean and covariance~\cite{eliason1993maximum}:
\begin{align} \label{eq:ML_estimator}
\Tilde{\mu}\!=\!\frac{1}{N}\sum_{i=1}^{N}\!x_i\!=\!\hat{\mu},~
\Tilde{\Sigma}\!=\!\frac{1}{N}\!\sum_{i=1}^{N} (x_i-\hat{\mu})(x_i-\hat{\mu})^T\!=\!\hat{\Sigma}.~~~~
\end{align}
By combining \eqref{eq:lin_probability} and \eqref{eq:ML_estimator} we get the desired results.
\end{proof}

The linear flow transformation $A$ can be obtained analytically by exploiting the spectral decomposition for the symmetric positive semi-definite (PSD) matrix $\hat{\Sigma}\!=\!QDQ^T$, where $Q$ is an orthogonal matrix whose columns are the eigen-vectors of $\hat{\Sigma}$ and $D$ is a diagonal matrix whose entries are its eigen-values. The resulting invertible linear flow transformation for data $X$ can be written as:
$$
X = AZ +b,~~~ Z = A^{-1} ( X - b),
$$
where $b = \hat{\mu}~,~A = Q D^{\frac{1}{2}},~A^{-1} = D^{-\frac{1}{2}} Q^T$.\footnote{To simplify notation, in the rest of this paper we assume that the empirical mean $\hat{\mu}$ is zero, achieved in practice by zero-centering the data.} In the sequel, we propose an extension of the linear flow that adds non-linear components, which we term a residual flow model.

\subsection{Residual Flow Model}\label{subsec:residual_flow_model}
In this section, we describe how to extend the linear flow model to include non-linear components. Rather than directly using a fully non-linear model like RealNVP or GLOW, as described in Section \ref{sec:background}, we would like a model that can be viewed and trained as an extension to the linear model. This approach will allow a principled improvement over the Gaussian model of Lee et al.~\cite{lee2018simple}, which we already know to perform well.

We begin by composing a linear flow with a residual flow model:
$$
    f^{res} = p_{k} \cdot f^{non-lin}_k \cdot p_{k-1} \dots p_2 \cdot f^{non-lin}_2 p_1 \cdot f^{non-lin}_1 \cdot A^{-1},
$$
with the following log determinant:
\begin{align*}
     \log\left( \left|\det\left(\frac{\partial f(x)}{\partial x^T}\right)\right|\right) = \log\left(\left|\det\left(A^{-1}\right)\right|\right) \\
     + \sum_{i}  \log\left( \left|\det\left(\frac{\partial f^{non-lin}_i(x)}{\partial x^T}\right)\right|\right).
\end{align*}
Note that, from Eq. \eqref{eq:3}, when $s_i$ and $t_i$ are set to zero, the non-linear terms $f^{non-lin}_i$ are reduced to the identity map. In this case, the permutation terms have no effect, as the components of $z$ have identical and independent distributions. Thus, in this case, the residual flow $f^{res}$ is equivalent to the linear flow $f^{lin} = A^{-1}$. Therefore, we can initialize the residual flow by fixing the networks $s_i$ and $t_i$ to be zero, and calculating $A$ as described in Section \ref{subsec:linear_flow_model}, which is equivalent to fitting a Gaussian distribution model to our data. Subsequently, we can fine-tune the non-linear components in the model to obtain a better fit to the data. In practice, setting only the last layer of the networks $s_i$ and $t_i$ to zero is sufficient for the initialization step.\footnote{ We found this to perform better in fine-tuning the non-linear terms, as most of the network is not initialized to zero and obtains large gradients in the initial training steps.}

Similar to the GLOW model \cite{kingma2018glow}, we found that the permutation terms $p_i$ have an important contribution, by diversifying the inputs of the non-linear components. In our implementation, we alternate between fixed, initially random,\footnote{The random permutation shuffles the preceding layer’s input in a predetermined random order that remains consistent throughout training.} permutation matrices and switch permutation matrices to mediate the non-linear flow blocks. Concretely, $p_i$ stands for a random permutation for odd $i$ and switch permutation for even $i$. Figure \ref{figure:Residual_flow_all} illustrates the proposed architecture, and the full implementation is described in Section \ref{sec:residual_flow_to_novalty_detec}.
\begin{figure} [t] \centering
\subfigure[Residual Flow blocks during initialization and training.]
{
\includegraphics[clip, trim=1.5cm 5.5cm 2cm 5cm, width=0.5\textwidth]{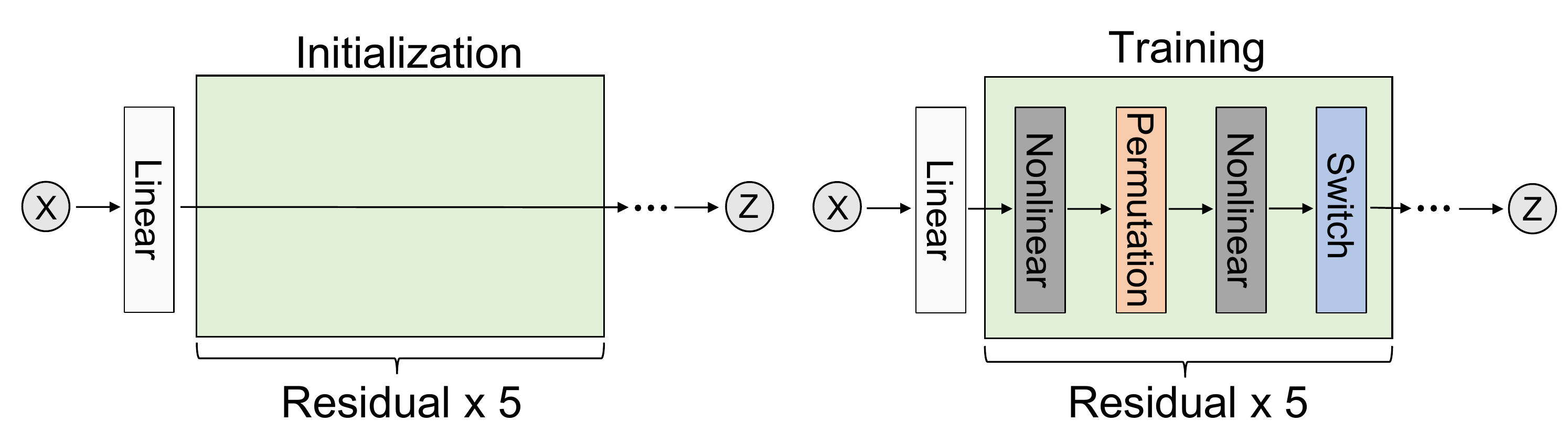}\label{figure:Residual_flow_diag}}
\,
\subfigure[The complete Residual Flow architecture $Z = f(X)$.]
{
\includegraphics[clip, trim=9cm 8cm 8.5cm 7cm, width=0.4\textwidth]{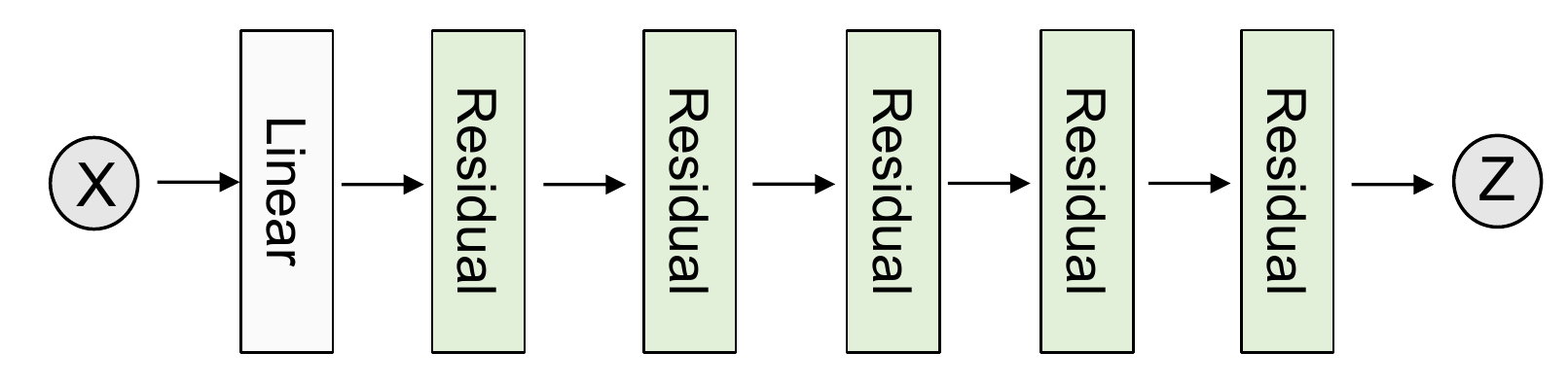} \label{figure:Residual_flow}}
\caption{Residual Flow architecture.}
\label{figure:Residual_flow_all}
\vspace{-0.4cm}
\end{figure}

\subsubsection{Degenerate case}\label{subsec:degeneracy}
If the covariance matrix $\hat{\Sigma}$ is not full rank, then the multivariate normal distribution is degenerate: its vector elements are linearly dependent, and the covariance matrix does not correspond to a density over the $d$-dimensional space. In this case, Lee et al.~\cite{lee2018simple} propose to use $\hat{\Sigma}^{\dagger}$, the pseudo-inverse of $\hat{\Sigma}$, to calculate the Mahalanobis distance:
\begin{align}\label{eq:mahalanobis2}
- \left(X - \hat{\mu}\right)^T \!\hat{\Sigma}^{\dagger} \!\left(X -\hat{\mu} \right),
\end{align}
which is equivalent to restricting attention to a subset of $k=rank(\hat{\Sigma})$ of the coordinates of $X$, such that the covariance matrix of this subset is positive definite (PS); the remaining coordinates are regarded as an affine function of the selected coordinates. 
In our model we handle degenerate distributions with a similar approach: We set $Z = A^{\dagger}X$ to be a $k$-dimensional vector with a $k$-dimensional Gaussian distribution, using a dimensionality reduction transformation $A^{\dagger} \in \mathbb{R}^{k \times d}$.\footnote{Note that here $A^{\dagger}$ is not the inverse of $A$.} We construct $A^{\dagger} = D^{-\frac{1}{2}}  Q^T$ with $D^{-\frac{1}{2}}\!\in\!\mathbb{R}^{k \times k}$ and $Q^T\!\in\!\mathbb{R}^{k \times d}$, by considering the inverse root of the $k$ non-zero eigen-values of $\hat{\Sigma}$ in $D^{-\frac{1}{2}}$ diagonal and their corresponding eigen-vectors in $Q^T$ rows. Note that using $A^{\dagger}$ for degenerated vectors $X$ yields the same Gaussian distribution as the pseudo-inverse used in \cite{lee2018simple}. In the rest of this paper we consider $A^{\dagger}$ as the linear flow transformation for degenerated vectors $X$.
After this linear dimensionality reduction, we apply the residual flow model on the resulting $k$-dimensional vector $Z$ as presented in Section \ref{subsec:residual_flow_model}.
As a remark, the aforementioned treatment removes only linear dependencies among feature elements, and does not address non-linear dependencies. Practically, however, we found that this approach is sufficient for all the experiments we conducted.

Our residual flow model is a general normalizing flow architecture, and we expect it to work well when the data approximately fits a Gaussian distribution. 

\subsection{Residual Flow Applied to OOD Detection}\label{sec:residual_flow_to_novalty_detec}

We now describe an application of the residual flow that extends the Gaussian model of \cite{lee2018simple} for OOD detection.
\noindent
First, for each network layer $l$, we extract the mean activation in the training data for each class label $\mu_{l,c}$. Then, for each sample $x$ in our training data, we extract the network activation in layer $l$, $\phi_l(x)$, and subtract from it the mean $\mu_{l,c}$ for the corresponding class, to obtain a centered feature training set $\hat{\phi}_l(x)$. 
Next, we fit a Gaussian distribution to the centered data by constructing a linear flow model for each layer as described in Section \ref{subsec:linear_flow_model}. We construct a single linear model \emph{for all} classes, similar to the single covariance matrix in \cite{lee2018simple}.
Finally, for each layer $l$, and for each class $c$, we train a residual flow model by training the non-linear flow blocks $f^{non-lin}_i$, as described in \ref{subsec:residual_flow_model}, and freeze the network weights in the linear block $f^{lin}_i$. As a stopping criteria for training the residual flow blocks, we use a separate validation set, and validate on the log-likelihood of the data. We found this approach to be effective for preventing overfitting in our experiments. This model, applied for OOD detection, already has good performance at the outset, leading to a better fit to the data distribution as training progresses.

\textbf{Implementation details:} 
We implement the model as a single linear flow block $f^{lin}\!=\!A^{-1}$, followed by $10$ non-linear flow blocks $f^{non-lin}$, producing a map $f^{res}$ totalling $11$ flow blocks. As for the layers $p_i$, which interconnect the blocks $f^{non-lin}$, we alternate between switch and random permutation matrices. We use three fully connected layers per non-linear block (in each $s_i$ and $t_i$) with leaky ReLU activation functions in the intermediate layers. We use a batch size of 256 and Adam \cite{Kingma2014AdamAM} optimizer for learning the non-linear blocks with a learning rate of $10^{-5} - 10^{-6}$, chosen via a separate validation set of $10\mathrm{K}$ examples.

\subsection{Input pre-processing}\label{subsec:Input_preprocessing}
Motivated by the success of input pre-processing in ODIN \cite{liang2017principled} and Mahalanobis \cite{lee2018simple}, we propose an extension of this idea to our approach. Since the Mahalanobis pre-processing can be seen as maximizing the likelihood of the input under the Gaussian model, we similarly introduce the following input pre-processing stage for our flow model:
\begin{align}\label{eq:7}
\Tilde{x} = x + \epsilon \cdot \sign\left( \nabla_x \log p(\phi_l(x); \hat{c})\right),
\end{align}
where $\hat{c} = \underset{c\,\in \,C}{\argmax}\,p(\phi_l(x); c)$ and $p(\phi_l(x); \hat{c})$ is the probability distribution of the feature space of the $l$-th layer of class $\hat{c}$, learned by our flow model. Note that this score aims to increase the probability of the in-distribution data.

\subsection{OOD Detection Algorithm}\label{subsec:Algorithm}
In this section we describe the proposed procedure for OOD detection. Using the training set, we first train a collection of residual flows for each layer and each class $\{ f^{res}_{l,c}: \forall l,c \}$ according to Section \ref{subsec:residual_flow_model}. Given a test example $x$, we extract the layers' activations for this example $\{\phi_l(x): \forall l\}$, and calculate the most probable class for each layer $\hat{c}_{l}$. Using $\hat{c}_{l}$ we calculate the pre-processed input $\Tilde{x}$, according to Eq.~\eqref{eq:7}, and re-calculate the layers' activations $\{\phi_l(\Tilde{x}): \forall l\}$. The probability of the most probable class serves as a score of the layer $S_l = \max \limits_c p_c\left( \phi_l(\Tilde{x}) - \mathbf{\widehat \mu}_{l,c}\right)$. Finally, the effective score is a weighted average of layers' scores $\sum_{l} \alpha_l S_l$. The weights are obtained using a similar strategy as in \cite{lee2018simple}, where the weights of the layers $\alpha_l$ are computed by training a logistic regression detector on a validation set. The full algorithm is detailed in Algorithm \ref{alg:computing_score}.

\subsection{Computational Overhead}\label{subsec:Complexity}
It is important to evaluate the computational overhead of using a more expressive model for network activations. We compare our method to \cite{lee2018simple}, and consider two cases: (i) During training: our initialization step is equivalent to the method of \cite{lee2018simple}. Thus, performance improvement comes at a cost of additional training time. Figure \ref{figure:AUROC_vs_iterations} shows the tradeoff between additional training iterations and performance gain. Note that the improvement monotonically increases with training iterations. (ii) During testing: In the test phase, both methods first calculate a forward pass of the test image through the classification network for feature extraction. Then, \cite{lee2018simple} calculates the Mahalanobis distance, while our method runs another forward pass of the residual flow networks. In our experiments, the forward pass of the classification network was the dominant complexity factor. This may change with a larger flow model, but in our experiments we did not require such. Thus, our performance advantage does not incur significant overhead.

\begin{algorithm}[t]
\caption{Computing the Residual-Flow score $S_l$.} \label{alg:computing_score}
\begin{algorithmic}
\State {\bf Input:} \small Test sample $x$, weights of logistic regression detector $\alpha_l$, noise $\varepsilon$ and $C$ residual-flow for each layer: $\{ f^{res}_{l,c}: \forall l,c \}$ 
\vspace{0.05in}
\hrule
\vspace{0.05in}
\State Initialize score vectors: $\mathbf{S}_{RF}(x) = [S_{l,c}:\forall l,c]$
\For{each layer $l \in 1,\ldots,L$}
\State Find the most probable class: 
\State  ~~~~~~~~~ $\widehat c = \arg \max_c ~p_c(\phi_{l}(x) - \mathbf{\widehat \mu}_{l,c})$
\State Add small noise to test sample: 
\State  ~~~~~~~~~ $\Tilde{x} = x + \varepsilon \text{sign} \bigtriangledown_x p_{\hat{c}} \left( \phi_l(x) - {\widehat \mu}_{l, \widehat c}\right)$ 
\State  Computing confidence score: 
\State ~~~~~~~~~ $S_l = \max \limits_c p_c\left( \phi_l(\Tilde{x}) - \mathbf{\widehat \mu}_{l,c}\right)$
\EndFor
\State \Return Confidence score for test sample {$\sum_{l} \alpha_l S_l$}
\end{algorithmic}
\end{algorithm}
\normalsize
\section{Related Work}
\label{sec:related_work}
OOD detection has mostly been studied in the unlabelled setting, where the data contains only samples (e.g., images) but not class labels. Classical methods include one-class SVM~\cite{Scholkopf:2001:ESH:1119748.1119749} and support vector data description~\cite{Tax:2004:SVD:960091.960109}, and more recently, deep learning methods have become popular~\cite{DBLP:journals/corr/abs-1901-03407}. Methods such as \cite{Erfani:2016:HLA:2952005.2952200,ae16,cao16hybrid,chen2017outlier} extract features using unsupervised learning techniques, and feed them to classical OOD detection methods. Deep SVDD~\cite{pmlr-v80-ruff18a} learns a neural-network encoding that minimizes the volume of data around a predetermined point in feature space. Recently, Golan and El-Yaniv~\cite{golan2018deep} proposed to learn features by applying a fixed set of geometric transformations to images, and training a deep network to classify which transformation was applied. Density estimation methods for detecting OOD examples have originally been studied in low dimensional space \cite{pimentel2014review, chow1970optimum, ghoting2008fast}. Recently, deep generative models such as generative adversarial networks, variational autoencoders, and deep energy-based models have been proposed for OOD detection in high-dimensional spaces \cite{an2015variational,suh2016echo,schlegl2017unsupervised,wang2017safer,DBLP:journals/corr/ZhaiCLZ16,Song2018LearningNR}.

Our work focuses on the labelled setting, where a network trained for image classification is provided, along with the training data and labels. Hendrycks and Gimpel~\cite{hendrycks2016baseline} proposes the soft-max output as a confidence score for OOD examples, and \cite{geifman2017selective} compared this approach with the Monte-Carlo dropout ensemble method. Liang et al.~\cite{liang2017principled} proposed ODIN, which combines temperature scaling and input pre-processing. The geometric transformations method of Golan and El-Yaniv~\cite{golan2018deep} can also be applied to the labelled setting. The state-of-the-art is the method of Lee et al.~\cite{lee2018simple} that uses the Mahalanobis distance in feature space. In our work we show that providing a better density model, leads to a marked improvement over Lee et al.'s results.

Concurrent with our work, several OpenReview postings suggested improvements to the method of \cite{lee2018simple}. Sastry et al.~\cite{sastry2020zeroshot} propose a scoring function for OOD detection based on the correlation between different features of the same layer, using higher-order Gram matrices, which can be seen as a different form of incorporating higher-order statistics beyond the Gaussian model. Yu et al.~\cite{yu2020outofdistribution} investigate the benefit of combining the global average of the feature maps with their spatial pattern information, while using the Gaussian model assumption. In principle, their approach can be combined with our improved flow-based density model. 

\section{Experiments}
\label{sec:experiments}

In our experiments, we aim to answer the following questions: (1) How does the residual flow model compare with conventional flow and Gaussian models? (2) How does our OOD detection method compare with state-of-the-art?

Our OOD detection evaluation follows the data sets and experiments in \cite{lee2018simple}, and consists of 3 training data sets: CIFAR10, CIFAR100, and SVHN, and 4 out-of-distribution (OOD) data sets: CIFAR10, Tiny ImageNet, SVHN, and LSUN. In the supplementary material we provide additional experiments, which draw a comparison between residual flow,  LDA (Mahalanobis) and the GDA model. The full residual flow implementation is available online.\footnote{\url{https://github.com/EvZissel/Residual-Flow}} 

\subsection{Residual Flow vs. Regular Flow}\label{subsec:residual_vs_regular}
In this section we compare the performance of learning a residual flow model over learning regular non-linear flow model.
First, we inspect the performance of the proposed approach on the task of distinguishing in- and out-of-distribution examples based on the first layer of ResNet, trained on CIFAR-100, where Tiny-ImageNet is used as OOD. In our comparison, we evaluate residual flow against regular non-linear flow and linear-flow/Mahalanobis density modeling. Figure \ref{figure:Roc_curve} presents a receiver operating characteristic (ROC) curve \cite{davis2006relationship} comparison of the three methods,\footnote{Training the flow models throughout this paper (residual and regular) is conducted using a validation set of 10K samples that are portioned from the training set, and the stopping criterion is the overfit set-point at which the validation likelihood ceases to increase.} demonstrating the superiority of the residual flow model in modeling feature layer distribution of a neural network. Next, in Figure \ref{figure:AUROC_vs_iterations}, we evaluate the area under the ROC (AUROC) curve as a function of training iterations. Note that the linear flow\footnote{The linear model is described in Supplementary material -- Section 2.}, as expected, converges to the same AUROC as the baseline Gaussian density model. The residual flow, however, starts at baseline performance (equivalent to the Gaussian model), and steadily improves upon it, as the non-linear components allow for better modelling of the data. The conventional non-linear flow, on the other hand, starts from a low AUROC score, rises erratically, and is not guaranteed to improve upon the baseline. The erratic behavior also makes it difficult to decide when to stop training. Indeed, we found this model to be much less stable in our evaluation.

\begin{figure*} [ht] \centering
\subfigure[ROC curve ResNet -- first layer]
{
\includegraphics[width=0.30\textwidth]{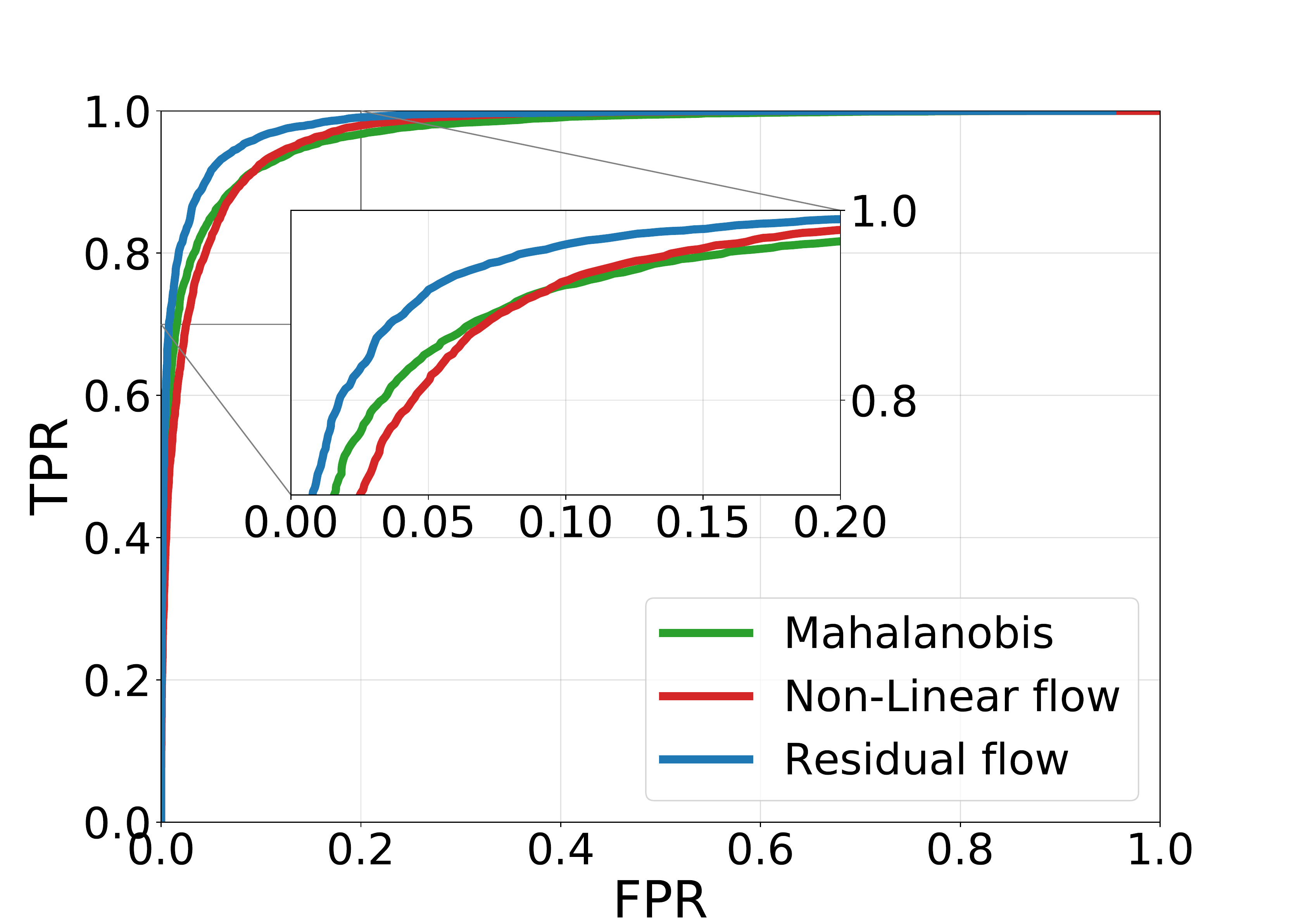}\label{figure:Roc_curve}}
\,
\subfigure[AUROC vs. Iterations]
{
\includegraphics[width=0.30\textwidth]{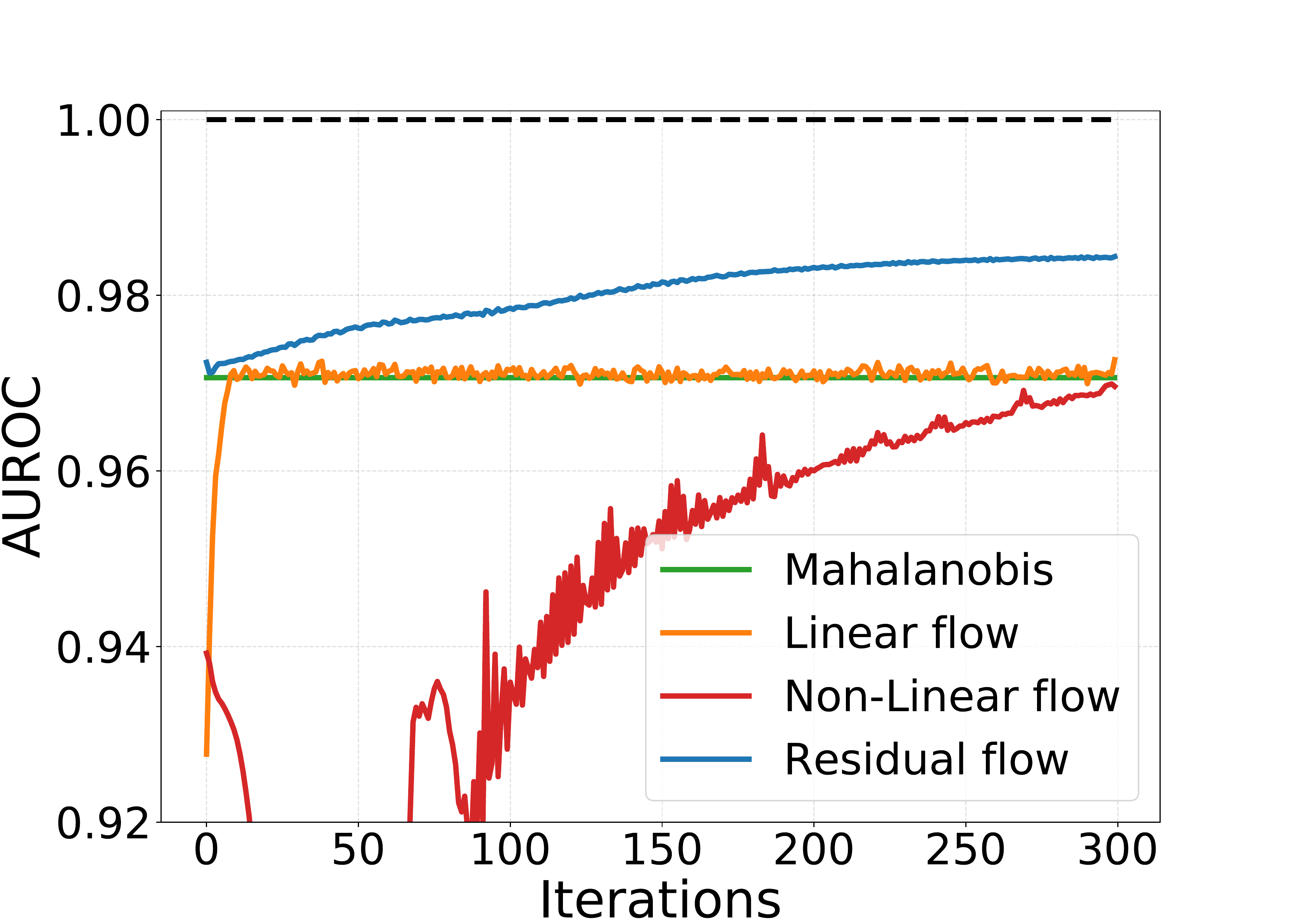} \label{figure:AUROC_vs_iterations}}
\subfigure[ROC curve DenseNet -- all layers]
{
\includegraphics[width=0.30\textwidth]{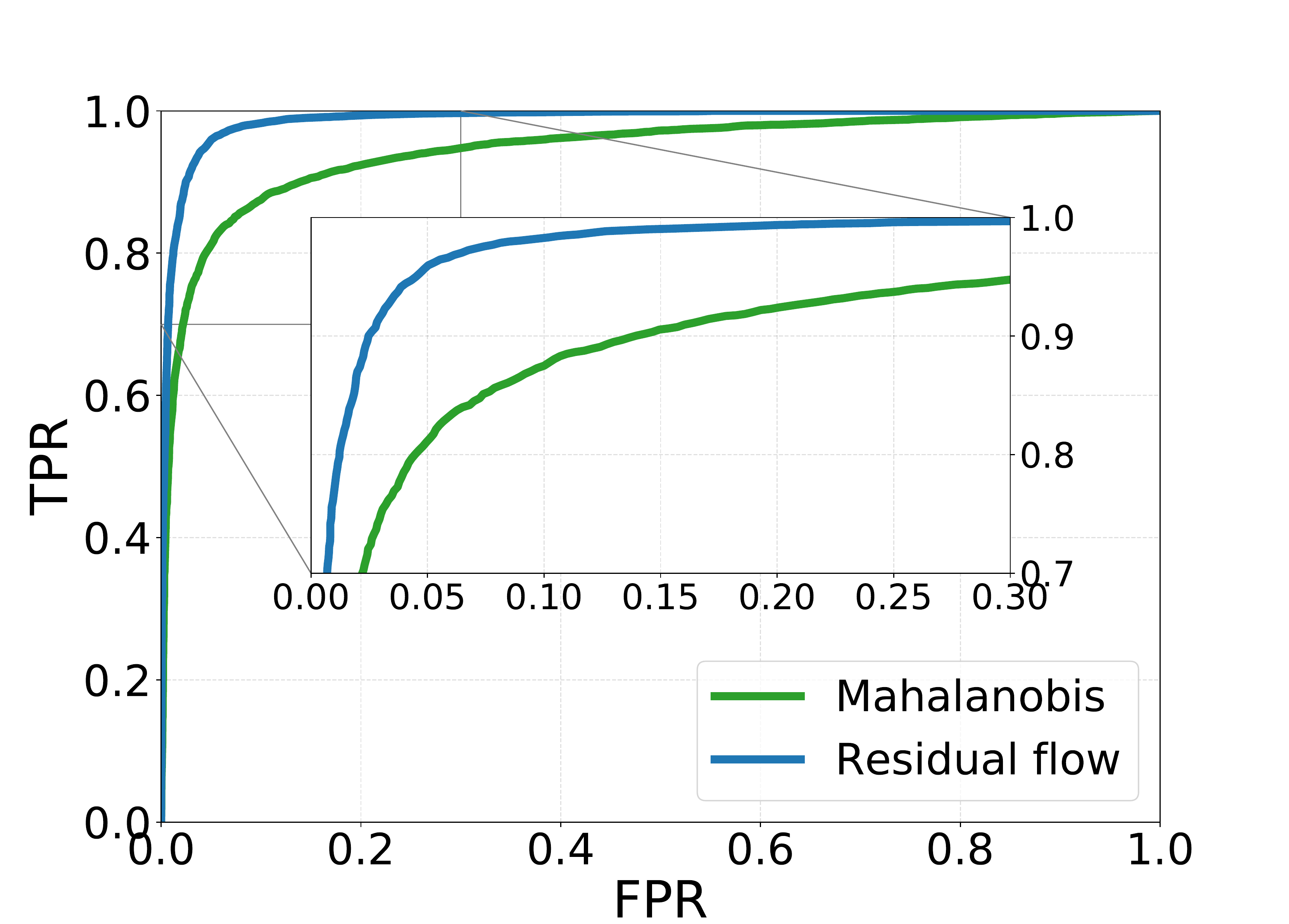}\label{figure:Roc_all}
}
\caption{(a) and (b) OOD detection using features taken from the first layer of ResNet trained on CIFAR-100, with TinyImageNet as OOD.\newline (a) ROC curve comparison of Residual flow (AUROC = 98.4), Non-linear flow (AUROC = 97.0) and Mahalanobis (AUROC = 97.0) \cite{lee2018simple}. (b) AUROC comparison as a function of training iterations for different models. Note that by our initialization method, the residual flow starts at baseline performance of Mahalanobis. (c) ROC comparison of Mahalanobis (AUROC = 94.6) and Residual flow (AUROC = 98.9), using a weighted average of the features taken from layers of DenseNet trained on CIFAR-100, with LSUN as OOD.}
\vspace{-0.35cm}
\end{figure*}

\subsection{OOD Detection Evaluation}\label{subsec:evaluation}
We conduct a series of experiments to evaluate the performance in detecting out-of-distribution examples. These tests are used by contemporary state-of-the-art methods \cite{hendrycks2016baseline,liang2017principled,lee2018simple} to benchmark the efficacy of an algorithm in distinguishing abnormalities.
We follow the practices presented in \cite{lee2018simple}, in which already-trained neural networks are used in conjunction with conventional datasets. 
The experiments use DenseNet with 100 layers  \cite{huang2017densely} and ResNet with 34 layers \cite{he2016deep} as target networks, trained on one of the following datasets: CIFAR-10, CIFAR-100 \cite{krizhevsky2009learning} and SVHN \cite{netzer2011reading}. 
Feature extraction is performed as proposed by Lee et al. \cite{lee2018simple}: At the outset, we extract the output of specific layers from the target network and average over the spatial domain to produce a set of 1-dimensional feature vectors, whose size matches the number of feature maps in the corresponding layer. The selected layers are the terminal layers of every dense-block (or residual-block) of DenseNet (or ResNet). Next, we train a set of residual flow networks, each observing a different output layer of the target network (e.g. DenseNet) activated by an entire class of examples from its original dataset. A portion of the training set, 10K in total, is set aside as a validation set, to prevent overfit during training. The process repeats for all classes and for all end-block layers of the target network, yielding a set of trained residual flows. At the test phase, a score is calculated for every layer of the target network and the final confidence score is obtained using weights produced by training a logistic regression detector (see Algorithm \ref{alg:computing_score}). 

The weights of the logistic regression decoder and the input pre-processing parameter, $\epsilon$, are the hyperparameters of our model, tuned using a separate validation set of in- (positive class) and out-of-distribution (negative class) pairs, consisting of 1,000 images of each class. 
Similarly to Lee et al. \cite{lee2018simple}, we also investigate performance when a validation set of OOD samples is not available, 
and in this case we tune the hyperparameters using validation sets of both in-distribution samples and corresponding adversarial samples generated by FGSM \cite{Goodfellow2014ExplainingAH} as out-of-distribution samples.

The networks are tested using their original test set, with the introduction of OOD samples from either LSUN \cite{yu2015lsun}, CIFAR-10 \cite{krizhevsky2009learning}, Tiny-ImageNet \cite{deng2009imagenet} or SVHN \cite{netzer2011reading}. The following performance measures are evaluated: true negative rate (TNR) at $95\%$ true positive rate (TPR), area under the receiver operating characteristic curve (AUROC), area under the precision-recall curve (AUPR), and detection accuracy. We compare our method to the state-of-the-art, which employs Mahalanobis score as a confidence score \cite{lee2018simple}. Note that to accommodate a fair comparison, we adopt the hyperparameter selection procedure presented in \cite{lee2018simple}.

Table \ref{tbl:out_main} aggregates the performance of our method compared to Mahalanobis for the task of OOD detection across all in- and out-of-distribution dataset pairs, when an OOD validation set is available. Table \ref{tbl:out_val_out_candidate3} compares the performance when the validation set is produced using FGSM, as described above. We present the detection performance measures of our method with and without input pre-processing (right and middle columns respectively), and compare it to Mahalanobis score method with input pre-processing (left column). 
Tables \ref{tbl:out_main} and \ref{tbl:out_val_out_candidate3} demonstrably show that our method surpasses the current state-of-the-art, significantly outperforming the Mahalanobis approach in some cases -- even without input pre-processing. For example, applying our method on ResNet trained on CIFAR-100 samples, when LSUN is used as OOD dataset, improves the AUROC from $66.2\%$ to $82.0\%$ (without input pre-processing) and $87.2\%$ (with input pre-processing). 
In summary, the results in tables \ref{tbl:out_main} and \ref{tbl:out_val_out_candidate3} demonstrate that
better modeling of feature activations leads to better OOD detection. Figure \ref{figure:Roc_all} further demonstrates the contribution of our method compared to Mahalanobis \cite{lee2018simple}. We produce a ROC curve using ResNet trained on CIFAR-100, with LSUN dataset used as OOD. Note that the performance in Figure \ref{figure:Roc_all} was obtained without any pre-processing of the data. As seen from Figure \ref{figure:Roc_all}, our method significantly outperforms the Mahalanobis score method. 

\begin{small}
\begin{table*}[ht] 
\centering
\resizebox{\textwidth}{!}{
\begin{tabular}{@{}ccclclclll@{}}
\toprule
\multirow{2}{*}{\begin{tabular}[c]{@{}c@{}} In-dist \\ (model) \end{tabular}}
 & \multirow{2}{*}{\begin{tabular}[c]{@{}c@{}} Out-of-dist\end{tabular}}
& \multicolumn{1}{c}{\begin{tabular}[c]{@{}c@{}} TNR at TPR 95\% \end{tabular}}
& \multicolumn{1}{c}{\begin{tabular}[c]{@{}c@{}} AUROC \end{tabular}}
& \multicolumn{1}{c}{\begin{tabular}[c]{@{}c@{}} Detection accuracy \end{tabular}} 
& \multicolumn{1}{c}{\begin{tabular}[c]{@{}c@{}} AUPR in \end{tabular}} 
& \multicolumn{1}{c}{\begin{tabular}[c]{@{}c@{}} AUPR out \end{tabular}} \\\cline{3-7}  
\multicolumn{1}{c}{} & \multicolumn{1}{c}{} & \multicolumn{5}{c}{  Mahalanobis \cite{lee2018simple}/ Res-Flow without pre-processing / Res-Flow with pre-processing} \\ \midrule
\multirow{3}{*}{\begin{tabular}[c]{@{}c@{}} CIFAR-10 \\(DenseNet) \end{tabular}} 
& \multicolumn{1}{c}{SVHN}  
& \multicolumn{1}{c}{85.8 / {\bf 94.9} / {\bf 94.9}}
& \multicolumn{1}{c}{96.6 / {\bf 98.9} / {\bf 98.9}}
& \multicolumn{1}{c}{91.9 / {\bf 95.3} / {\bf 95.3}}
& \multicolumn{1}{c}{98.7 / {\bf 99.5} / {\bf 99.5}}
& \multicolumn{1}{c}{88.8 / {\bf 97.5} / {\bf 97.5}}\\
& \multicolumn{1}{c}{ImageNet} 
& \multicolumn{1}{c}{95.3 / {\bf 96.4} / {\bf 96.4}}
& \multicolumn{1}{c}{98.9 / {\bf 99.2} / {\bf 99.2}}
& \multicolumn{1}{c}{95.2 / {\bf 96.0} / {\bf 96.0}}
& \multicolumn{1}{c}{98.9 / {\bf 99.2} / {\bf 99.2}}
& \multicolumn{1}{c}{98.7 / {\bf 99.2} / {\bf 99.2}}\\
& \multicolumn{1}{c}{LSUN}
& \multicolumn{1}{c}{97.9 / {\bf 98.2} / {\bf 98.2}}
& \multicolumn{1}{c}{99.3 / {\bf 99.5} / {\bf 99.5}}
& \multicolumn{1}{c}{96.8 / {\bf 97.1} / {\bf 97.1}}
& \multicolumn{1}{c}{99.3 / {\bf 99.6} / {\bf 99.6}}
& \multicolumn{1}{c}{98.2 / {\bf 99.5} / {\bf 99.5}}\\ \midrule
\multirow{3}{*}{\begin{tabular}[c]{@{}c@{}} CIFAR-100 \\(DenseNet) \end{tabular}} 
& \multicolumn{1}{c}{SVHN}  
& \multicolumn{1}{c}{82.9 / 73.0 / {\bf 84.9}}
& \multicolumn{1}{c}{96.1 / 95.2 / {\bf 97.5}}
& \multicolumn{1}{c}{90.9 / 88.7 / {\bf 91.9}}
& \multicolumn{1}{c}{98.5 / 97.5 / {\bf 99.0}}
& \multicolumn{1}{c}{89.0 / 91.1 / {\bf 95.1}}\\
& \multicolumn{1}{c}{TinyImageNet} 
& \multicolumn{1}{c}{85.8 / \bf{93.0}/ \bf{93.0}}
& \multicolumn{1}{c}{96.6 / \bf{98.5} / \bf{98.5}}
& \multicolumn{1}{c}{91.2 / \bf{94.1} / \bf{94.1}}
& \multicolumn{1}{c}{96.9 / \bf{98.5} / \bf{98.5}}
& \multicolumn{1}{c}{95.5 / \bf{98.5} / \bf{98.5}}\\
& \multicolumn{1}{c}{LSUN}
& \multicolumn{1}{c}{83.6 / {\bf 96.3} / {\bf 96.3}}
& \multicolumn{1}{c}{94.9 / {\bf98.9} / {\bf 98.9}}
& \multicolumn{1}{c}{89.9 / {\bf 95.7} / {\bf 95.7}}
& \multicolumn{1}{c}{95.7 / {\bf 99.0} / {\bf 99.0}}
& \multicolumn{1}{c}{93.0/ {\bf 98.8} / {\bf 98.8}}\\ \midrule
\multirow{3}{*}{\begin{tabular}[c]{@{}c@{}} SVHN \\(DenseNet) \end{tabular}} 
& \multicolumn{1}{c}{CIFAR-10}  
& \multicolumn{1}{c}{96.5 / {\bf 99.0} / {\bf 99.0}}
& \multicolumn{1}{c}{98.9 / {\bf 99.5} / {\bf 99.5}}
& \multicolumn{1}{c}{95.9 / {\bf 97.4} / {\bf 97.4}}
& \multicolumn{1}{c}{95.6 / {\bf 97.8} / {\bf 97.8}}
& \multicolumn{1}{c}{99.6 / {\bf 99.8} / {\bf 99.8}}\\
& \multicolumn{1}{c}{TinyImageNet}
& \multicolumn{1}{c}{99.8 / {\bf 100.0} / {\bf 100.0}}
& \multicolumn{1}{c}{99.9 / {\bf 100.0} / {\bf 100.0}}
& \multicolumn{1}{c}{98.8 / {\bf 99.4} / {\bf 99.4}}
& \multicolumn{1}{c}{99.6 / {\bf 99.8} / {\bf 99.8}}
& \multicolumn{1}{c}{{\bf 100.0} / {\bf 100.0} / {\bf 100.0}}\\
& \multicolumn{1}{c}{LSUN} 
& \multicolumn{1}{c}{{\bf100.0}/ {\bf 100.00}  / {\bf 100.00 }}
& \multicolumn{1}{c}{99.9 / {\bf100.0 }/ {\bf 100.0}}
& \multicolumn{1}{c}{99.3 / {\bf 99.7} / {\bf 99.7}}
& \multicolumn{1}{c}{99.7 / {\bf99.9 }/ {\bf 99.9}}
& \multicolumn{1}{c}{{\bf100.0} / {\bf100.0} / {\bf 100.0}}\\ \midrule
\multirow{3}{*}{\begin{tabular}[c]{@{}c@{}} CIFAR-10 \\(ResNet) \end{tabular}} 
& \multicolumn{1}{c}{SVHN}  
& \multicolumn{1}{c}{96.4 / 94.5  / {\bf 96.5}}
& \multicolumn{1}{c}{{\bf 99.1} / 98.9 / {\bf 99.1}}
& \multicolumn{1}{c}{{\bf 95.8} / 94.9 / {\bf 95.8}}
& \multicolumn{1}{c}{{\bf 99.6} / {\bf 99.6} / {\bf 99.6}}
& \multicolumn{1}{c}{{\bf 98.3} / 97.6 / {\bf 98.3}}\\
& \multicolumn{1}{c}{TinyImageNet} 
& \multicolumn{1}{c}{97.1 / {\bf 97.8} / {\bf 97.8}}
& \multicolumn{1}{c}{99.5 / {\bf 99.6} / {\bf 99.6}}
& \multicolumn{1}{c}{96.3 / {\bf 96.9} / {\bf 96.9}}
& \multicolumn{1}{c}{99.5 / {\bf 99.6} / {\bf 99.6}}
& \multicolumn{1}{c}{99.5 / {\bf 99.6} / {\bf 99.6}}\\
& \multicolumn{1}{c}{LSUN}  
& \multicolumn{1}{c}{98.9 / {\bf 99.0} / {\bf 99.0}}
& \multicolumn{1}{c}{99.7 / {\bf 99.8} / {\bf 99.8}}
& \multicolumn{1}{c}{97.7 / {\bf 97.8} / {\bf 97.8}}
& \multicolumn{1}{c}{99.7 / {\bf 99.8} / {\bf 99.8}}
& \multicolumn{1}{c}{99.7 / {\bf 99.8} / {\bf 99.8}}\\ \midrule
\multirow{3}{*}{\begin{tabular}[c]{@{}c@{}} CIFAR-100 \\(ResNet) \end{tabular}} 
& \multicolumn{1}{c}{SVHN} 
& \multicolumn{1}{c}{92.0 / 88.8 / {\bf 93.0}}
& \multicolumn{1}{c}{98.4 / 97.8 / {\bf 98.5}}
& \multicolumn{1}{c}{93.7 / 92.6 / {\bf 94.5}}
& \multicolumn{1}{c}{99.3 / 99.1 / {\bf 99.3}}
& \multicolumn{1}{c}{96.4 / 95.3 / {\bf 97.1}}\\
& \multicolumn{1}{c}{TinyImageNet} 
& \multicolumn{1}{c}{90.8 / {\bf 95.0} / 94.6}
& \multicolumn{1}{c}{98.2 / {\bf 98.9} / {\bf 98.9}}
& \multicolumn{1}{c}{93.3 / {\bf 95.0} / {\bf 95.0}}
& \multicolumn{1}{c}{98.1 / {\bf 98.9} / {\bf 98.9}}
& \multicolumn{1}{c}{98.2 / {\bf 98.9} / 98.8}\\
& \multicolumn{1}{c}{LSUN}
& \multicolumn{1}{c}{90.9 /{\bf 96.7} / 96.2}
& \multicolumn{1}{c}{98.2 / {\bf 99.1} / 99.0}
& \multicolumn{1}{c}{93.5 / 96.0 / {\bf 95.7}}
& \multicolumn{1}{c}{97.8 / {\bf 99.0} / 98.9}
& \multicolumn{1}{c}{98.4 / {\bf 98.8} / 98.6}\\ \midrule
\multirow{3}{*}{\begin{tabular}[c]{@{}c@{}} SVHN \\(ResNet) \end{tabular}} 
& \multicolumn{1}{c}{CIFAR-10}
& \multicolumn{1}{c}{98.5 / 99.3 / {\bf  99.4}}
& \multicolumn{1}{c}{99.3 / {\bf 99.6} /{\bf  99.6}}
& \multicolumn{1}{c}{96.9 / {\bf 97.7} / {\bf 97.7}}
& \multicolumn{1}{c}{97.0 / {\bf 98.3} / {\bf 98.3}}
& \multicolumn{1}{c}{99.7 / {\bf 99.9} / {\bf 99.9}}\\
& \multicolumn{1}{c}{TinyImageNet} 
& \multicolumn{1}{c}{99.9 / {\bf 100.0 }/ {\bf 100.0}}
& \multicolumn{1}{c}{99.9 / {\bf 100.0}  / 99.9}
& \multicolumn{1}{c}{99.1 / {\bf 99.5} / {\bf 99.3}}
& \multicolumn{1}{c}{99.1 / {\bf 99.8} / 99.7}
& \multicolumn{1}{c}{99.9 / {\bf 100.0} / {\bf 100.0}}\\
& \multicolumn{1}{c}{LSUN}
& \multicolumn{1}{c}{99.9 / {\bf 100.0} / {\bf 100.0}}
& \multicolumn{1}{c}{99.9 / {\bf 100.0} / {\bf 100.0}}
& \multicolumn{1}{c}{99.5 / {\bf 99.7} / {\bf 99.7}}
& \multicolumn{1}{c}{99.2 / {\bf 99.8} / {\bf 99.8}}
& \multicolumn{1}{c}{99.9 / {\bf 100.0} / {\bf 100.0}}\\ 
\bottomrule
\end{tabular}}
\vspace{+0.02in}
\caption{A comparison between our method and Mahalanobis \cite{lee2018simple} on the task of out-of-distribution detection for image classification of various in- and out-of-distribution data sets. The hyper-parameters were tuned using a validation set of in- and out-of-distribution datasets. The values presented here are percentages and the best results are indicated in bold.}
\label{tbl:out_main}
\vspace{-0.1cm}
\end{table*}
\end{small}

\begin{small}
\begin{table*}[ht] 
\centering
\resizebox{\textwidth}{!}{
\begin{tabular}{@{}ccclclclll@{}}
\toprule
\multirow{2}{*}{\begin{tabular}[c]{@{}c@{}} In-dist \\ (model) \end{tabular}}
 & \multirow{2}{*}{\begin{tabular}[c]{@{}c@{}} Out-of-dist\end{tabular}}
& \multicolumn{1}{c}{\begin{tabular}[c]{@{}c@{}} TNR at TPR 95\% \end{tabular}}
& \multicolumn{1}{c}{\begin{tabular}[c]{@{}c@{}} AUROC \end{tabular}}
& \multicolumn{1}{c}{\begin{tabular}[c]{@{}c@{}} Detection accuracy \end{tabular}} 
& \multicolumn{1}{c}{\begin{tabular}[c]{@{}c@{}} AUPR in \end{tabular}} 
& \multicolumn{1}{c}{\begin{tabular}[c]{@{}c@{}} AUPR out \end{tabular}} \\\cline{3-7}  
\multicolumn{1}{c}{} & \multicolumn{1}{c}{} & \multicolumn{5}{c}{ Mahalanobis \cite{lee2018simple}/ Res-Flow without pre-processing / Res-Flow with pre-processing} \\ \midrule
\multirow{3}{*}{\begin{tabular}[c]{@{}c@{}} CIFAR-10 \\(DenseNet) \end{tabular}} 
& \multicolumn{1}{c}{SVHN}
& \multicolumn{1}{c}{88.7 / {\bf 91.3}/ 86.1 }
& \multicolumn{1}{c}{97.6 / {\bf 98.3} / 97.3 }
& \multicolumn{1}{c}{92.4 / {\bf 93.8} / 91.6 }
& \multicolumn{1}{c}{94.7 / {\bf 96.6} / 94.3 }
& \multicolumn{1}{c}{99.0 / {\bf 99.3} / 99.0 }\\
& \multicolumn{1}{c}{TinyImageNet} 
& \multicolumn{1}{c}{88.6 / 96.0 / {\bf 96.1} }
& \multicolumn{1}{c}{97.5 / {\bf 99.1} / {\bf 99.1} }
& \multicolumn{1}{c}{92.2 / {\bf 95.6} / {\bf 95.6} }
& \multicolumn{1}{c}{97.4 / {\bf 99.1} / {\bf 99.1} }
& \multicolumn{1}{c}{97.7 / {\bf 99.2} / {\bf 99.2} }\\
& \multicolumn{1}{c}{LSUN} 
& \multicolumn{1}{c}{92.4 / 98.0 / {\bf 98.1} }
& \multicolumn{1}{c}{98.3 / {\bf 99.5} / {\bf 99.5} }
& \multicolumn{1}{c}{93.9 / 96.7 / {\bf 96.9} }
& \multicolumn{1}{c}{98.4 / {\bf 99.5} / {\bf 99.5} }
& \multicolumn{1}{c}{98.2 / 99.4 / {\bf 99.5} }\\ \midrule
\multirow{3}{*}{\begin{tabular}[c]{@{}c@{}} CIFAR-100 \\(DenseNet) \end{tabular}} 
& \multicolumn{1}{c}{SVHN} 
& \multicolumn{1}{c}{48.7 /{\bf 59.8} / 48.9 }
& \multicolumn{1}{c}{85.6 / {\bf 91.4} / 87.9}
& \multicolumn{1}{c}{80.0 / {\bf 83.7} / 80.0 }
& \multicolumn{1}{c}{63.7 /{\bf 82.9}/  74.9}
& \multicolumn{1}{c}{93.3 /{\bf 96.1} /  94.3}\\
& \multicolumn{1}{c}{TinyImageNet} 
& \multicolumn{1}{c}{80.4 / {\bf 91.7} / 91.5 }
& \multicolumn{1}{c}{92.7 / {\bf 98.3} / 98.1 }
& \multicolumn{1}{c}{88.0 / {\bf 93.6} / 93.4 }
& \multicolumn{1}{c}{87.4 / {\bf 98.3} / 98.0 }
& \multicolumn{1}{c}{94.5 / {\bf 98.4} / 98.3 }\\
& \multicolumn{1}{c}{LSUN}
& \multicolumn{1}{c}{83.8 / 95.4 / {\bf 95.8 }}
& \multicolumn{1}{c}{95.0 / {\bf 98.9} / {\bf 98.9 }}
& \multicolumn{1}{c}{90.0 / 95.3 / {\bf 95.4 }}
& \multicolumn{1}{c}{93.0 / {\bf 99.0} / 98.9 }
& \multicolumn{1}{c}{95.7 / {\bf 98.8} / {\bf 98.8 }}\\ \midrule
\multirow{3}{*}{\begin{tabular}[c]{@{}c@{}} SVHN \\(DenseNet) \end{tabular}} 
& \multicolumn{1}{c}{CIFAR-10} 
& \multicolumn{1}{c}{92.5 / {\bf 95.1} / 90.0 }
& \multicolumn{1}{c}{96.7 / {\bf 98.7} / 98.0}
& \multicolumn{1}{c}{93.8 / {\bf 95.3} / 93.4}
& \multicolumn{1}{c}{97.9 / 99.6 / {\bf 99.7}}
& \multicolumn{1}{c}{93.5 / {\bf 95.2} / 93.6}\\
& \multicolumn{1}{c}{TinyImageNet}   
& \multicolumn{1}{c}{99.1 / 99.7 / {\bf 99.9 }}
& \multicolumn{1}{c}{99.5 / {\bf 99.9} / {\bf 99.9}}
& \multicolumn{1}{c}{98.7 / {\bf 99.2} / 99.0}
& \multicolumn{1}{c}{99.6 / {\bf 100.0} / {\bf 100.0}}
& \multicolumn{1}{c}{99.2 / {\bf 99.8} / 99.6}\\
& \multicolumn{1}{c}{LSUN}   
& \multicolumn{1}{c}{99.7 / {\bf 100.0} / {\bf 100.0}}
& \multicolumn{1}{c}{99.8 / {\bf 100.0} / 99.9 }
& \multicolumn{1}{c}{99.1 / {\bf 99.5} / 99.4}
& \multicolumn{1}{c}{99.9 / 100.0 / {\bf 100.0}}
& \multicolumn{1}{c}{99.6 / {\bf 99.8} / 99.7 }\\ \midrule
\multirow{3}{*}{\begin{tabular}[c]{@{}c@{}} CIFAR-10 \\(ResNet) \end{tabular}} 
& \multicolumn{1}{c}{SVHN} 
& \multicolumn{1}{c}{87.5 / {\bf 91.0} / {\bf 91.0}}
& \multicolumn{1}{c}{97.4 / {\bf 98.2} / {\bf 98.2}}
& \multicolumn{1}{c}{91.8 / {\bf 93.8} / {\bf 93.8}}
& \multicolumn{1}{c}{93.8 / {\bf 96.6} / {\bf 96.6}}
& \multicolumn{1}{c}{98.9 / {\bf 99.1} / {\bf 99.1}}\\
& \multicolumn{1}{c}{TinyIageNet} 
& \multicolumn{1}{c}{93.1 / {\bf 98.0} / {\bf 98.0}}
& \multicolumn{1}{c}{97.9 / {\bf 99.6} / {\bf 99.6}}
& \multicolumn{1}{c}{94.1 / {\bf 97.0} / {\bf 97.0}}
& \multicolumn{1}{c}{95.4 / {\bf 99.6} / {\bf 99.6}}
& \multicolumn{1}{c}{98.4 / {\bf 99.6} / {\bf 99.6}}\\
& \multicolumn{1}{c}{LSUN} 
& \multicolumn{1}{c}{97.0 / {\bf 99.1} / {\bf 99.1}}
& \multicolumn{1}{c}{99.2 / {\bf 99.8} / {\bf 99.8}}
& \multicolumn{1}{c}{96.3 / {\bf 98.0} / {\bf 98.0}}
& \multicolumn{1}{c}{98.6 / {\bf 99.8} / {\bf 99.8}}
& \multicolumn{1}{c}{99.3 / {\bf 99.8} / {\bf 99.8}}\\ \midrule
\multirow{3}{*}{\begin{tabular}[c]{@{}c@{}} CIFAR-100 \\(ResNet) \end{tabular}}
& \multicolumn{1}{c}{SVHN}
& \multicolumn{1}{c}{66.5 / 57.2 / {\bf 74.1}}
& \multicolumn{1}{c}{93.2 / 90.7 / {\bf 95.1}}
& \multicolumn{1}{c}{85.9 / 83.8 / {\bf 88.7}}
& \multicolumn{1}{c}{86.4 / 80.5 / {\bf 90.4}}
& \multicolumn{1}{c}{96.6 / 95.4 / {\bf 97.5}}\\
& \multicolumn{1}{c}{TinyImageNet} 
& \multicolumn{1}{c}{56.7 / 71.6 / {\bf 77.5}}
& \multicolumn{1}{c}{76.9 / 86.8 / {\bf 90.1}}
& \multicolumn{1}{c}{77.6 / 84.3 / {\bf 87.1}}
& \multicolumn{1}{c}{63.0 / 74.8 / {\bf 79.6}}
& \multicolumn{1}{c}{83.7 / 90.4 / {\bf 93.1}}\\
& \multicolumn{1}{c}{LSUN}
& \multicolumn{1}{c}{38.4 / 61.1 / {\bf 70.4}}
& \multicolumn{1}{c}{66.2 / 82.0 / {\bf 87.2}}
& \multicolumn{1}{c}{69.5 / 80.1 / {\bf 84.1}}
& \multicolumn{1}{c}{54.6 / 70.0 / {\bf 75.9}}
& \multicolumn{1}{c}{73.9 / 86.5 / {\bf 90.5}}\\ \midrule
\multirow{3}{*}{\begin{tabular}[c]{@{}c@{}} SVHN \\(ResNet) \end{tabular}} 
& \multicolumn{1}{c}{CIFAR-10} 
& \multicolumn{1}{c}{95.2 / {\bf 97.1} / 96.6}
& \multicolumn{1}{c}{98.1 / {\bf 99.1} / 99.0}
& \multicolumn{1}{c}{95.2 / {\bf 96.1} / 95.8}
& \multicolumn{1}{c}{98.5 / {\bf 99.7} / {\bf 99.7}}
& \multicolumn{1}{c}{95.2 / {\bf 96.7} / 96.5}\\
& \multicolumn{1}{c}{TinyImageNet} 
& \multicolumn{1}{c}{99.3 /  {\bf 99.9} / {\bf 99.9}}
& \multicolumn{1}{c}{99.4 / {\bf 99.9 } / {\bf 99.9}}
& \multicolumn{1}{c}{98.9 / {\bf 99.3 } / 99.2}
& \multicolumn{1}{c}{98.9 / {\bf 99.9} / {\bf 99.9}}
& \multicolumn{1}{c}{98.3 / {\bf 99.7} / {\bf 99.7}}\\
& \multicolumn{1}{c}{LSUN}
& \multicolumn{1}{c}{99.9 / {\bf 100.0} / {\bf  100.0}}
& \multicolumn{1}{c}{99.9 / {\bf 100.0} / {\bf 100.0}}
& \multicolumn{1}{c}{99.5 / {\bf 99.7} / 99.6 }
& \multicolumn{1}{c}{99.9 / {\bf 100.0} / {\bf 100.0}}
& \multicolumn{1}{c}{98.8 / {\bf 99.7} / {\bf 99.7}}\\ 
\bottomrule
\end{tabular}}
\vspace{+0.02in}
\caption{A comparison between our method and Mahalanobis \cite{lee2018simple} on the task of out-of-distribution detection for image classification of various in- and out-of-distribution data sets. The hyper-parameters were tuned using strictly in-distribution and adversarial (FGSM) samples. The values presented here are percentages and the best results are indicated in bold.}
\label{tbl:out_val_out_candidate3}
\end{table*}
\end{small}

\section{Conclusions}\label{sec:Conclusions}
We proposed an efficient method for detecting out-of-distribution inputs for trained neural networks, without retraining the network or modifying its underlying architecture, nor compromising its classification accuracy on in-distribution data. Key to our approach is a novel deep generative model -- the residual flow, which is a principled extension of a Gaussian distribution model using a non-linear normalizing flow. This model, which is of independent interest, is most suitable for modelling distributions that are approximately Gaussian. Our method is general, and in principle can be applied to various data such as speech recognition and natural language processing. On deep networks trained for image classification, we obtain state-of-the-art out-of-distribution detection performance. 

\section{Acknowledgments}\label{sec:Acknowledgments}
This work is partly funded by the Israel Science Foundation (ISF-759/19) and the Open Philanthropy Project Fund, an advised fund of Silicon Valley Community Foundation.


{\small
\bibliographystyle{ieee_fullname}
\bibliography{Residual_Flow}
}

\clearpage
\setcounter{section}{0}

\noindent
\textbf{
\Large Supplementary Materials} 
\vspace{0.5cm}

\noindent This Supplementary material elaborates on the Residual flow algorithm and provides additional experiments.

\section{Comparison: Proposed approach vs. LDA (Mahalanobis) and GDA models}
In this section we examine the performance of our approach compared with LDA (Mahalanobis) and GDA models. In GDA, feature activations of neural networks are modeled using Gaussian discriminant analysis, i.e.~posterior of a Gaussian distribution with different mean and different covariance matrix for each class. Calculating the log-likelihood of this model is equivalent to measuring the Mahalanobis distance using a different covariance matrix for each class and adding to it the log-determinant of the class's precision matrix\footnote{We also compare our method to a GDA variant, which uses the per-class covariance matrix without the contribution of the log determinant of the precision matrix. The results are similar to those of the full GDA model, shown in Figures \ref{fig:AUROC_Densenet} and \ref{fig:AUROC_ResNet}.}. As in Section \ref{subsec:degeneracy}, if the feature vector is degenerate, we restrict our attention to its corresponding non-degenerate sub-vector. In LDA (Mahalanobis), the feature activations are modeled using linear discriminant analysis, i.e.~posterior of a Gaussian distribution with different mean but with an identical covariance matrix for all classes. We compare these models without employing input-preprocessing. Figure \ref{fig:DenseNet_roc} compares the performance of Residual Flow against LDA and GDA for the task of OOD detection. The models use ResNet trained on CIFAR-100 (in-distribution) and tested on various OOD datasets. The Figure shows that our method consistently improves upon the state-of-the-art (LDA model). Note that GDA may produce inferior results in some cases. Figures \ref{fig:AUROC_Densenet} and \ref{fig:AUROC_ResNet} show the AUROC comparison on various in- and out-of-distribution datasets of DenseNet and ResNet, respectively. The Figures affirm the observation that modeling feature activations with GDA can deteriorate performance in some cases, especially when the number of per-class training examples is limited - as in the case of CIFAR-100 (Figure \ref{fig:AUROC_Resnet_CIFAR100}). Estimating the empirical covariance matrix for each class (GDA) suffers from high variance, exacerbated in scenarios of a small training set. By learning the residual from the LDA model, our method overcomes this limitation, resulting in consistently superior performance over stat-of-the-art. 


\begin{figure} [ht] \centering
\subfigure[Residual Flow blocks in initialization and training.]
{
\includegraphics[clip, trim=0cm 6cm 0cm 5cm, width=0.5\textwidth]{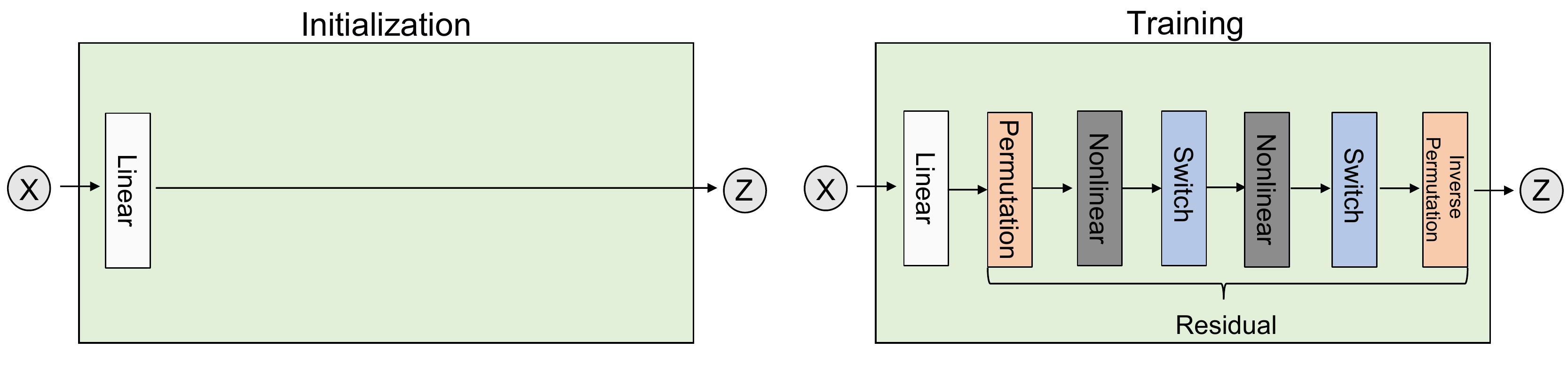}\label{figure:alternative_arc_flow_blocks}}
\,
\subfigure[The complete Residual Flow architecture $Z = f(X)$.]
{
\includegraphics[clip, trim=5cm 8cm 6cm 7cm, width=0.4\textwidth]{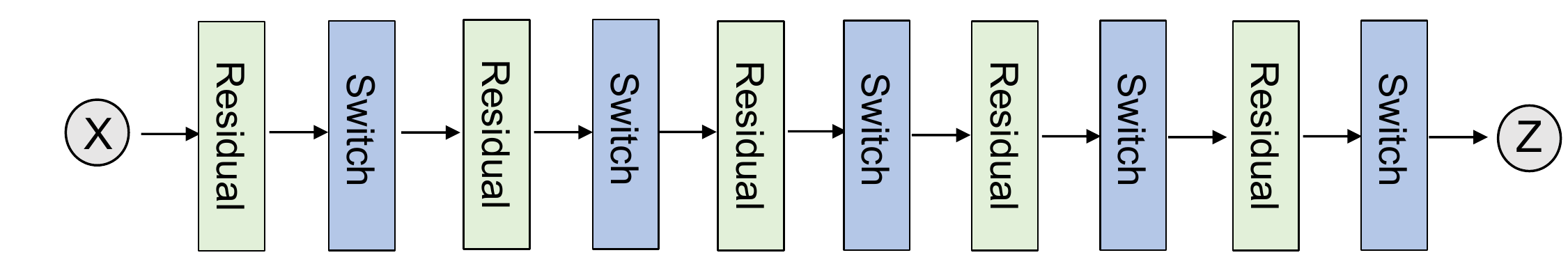} \label{figure:alternative_arc_flow}}
\caption{Residual Flow alternative architecture.}
\label{figure:alternative_arc}
\end{figure}

\begin{figure*} [ht] \centering
\subfigure[]
{
\includegraphics[width=0.31\textwidth]{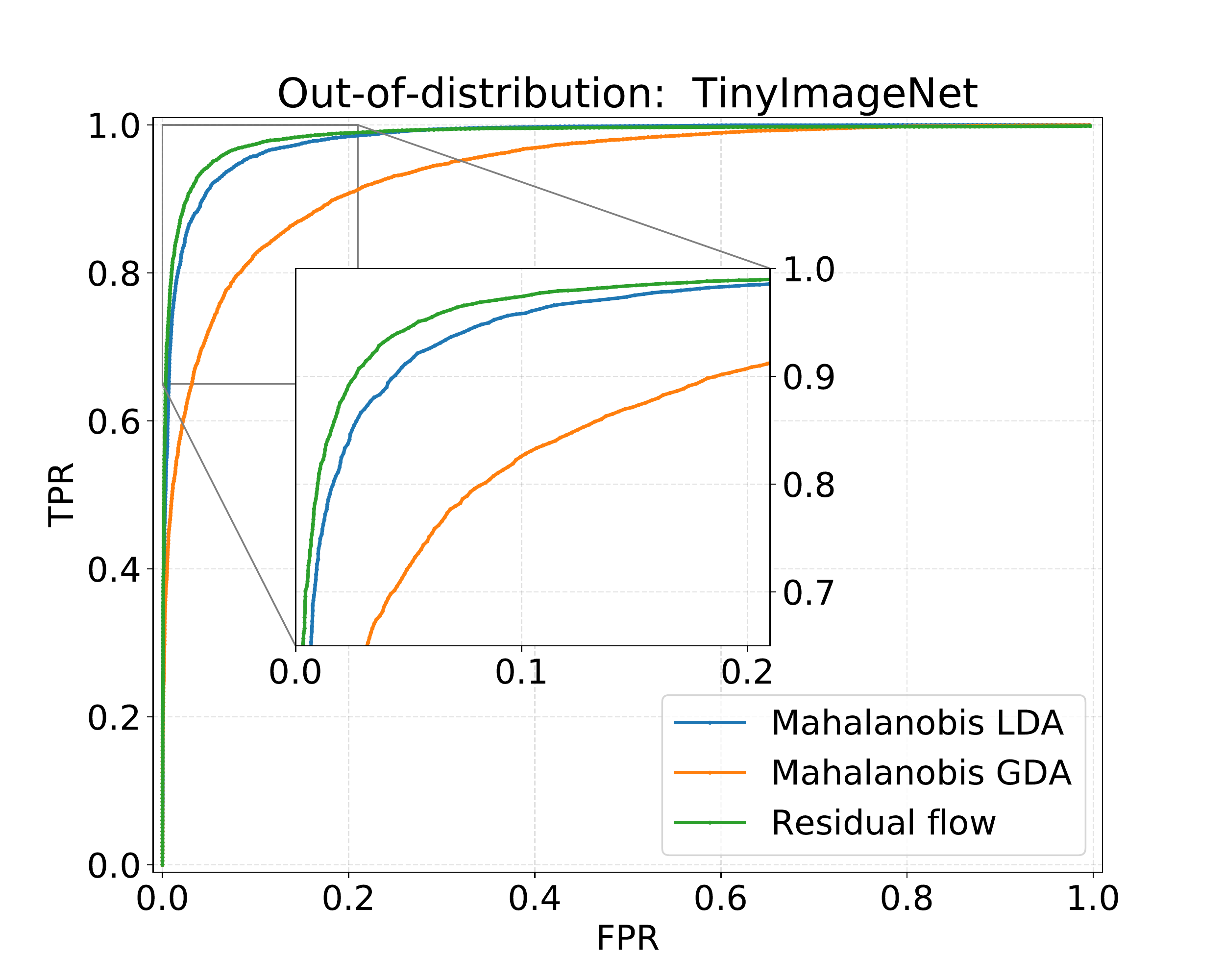} \label{fig:Resnet_CIFAR100_Iamgenet}} 
\,
\subfigure[]
{
\includegraphics[width=0.31\textwidth]{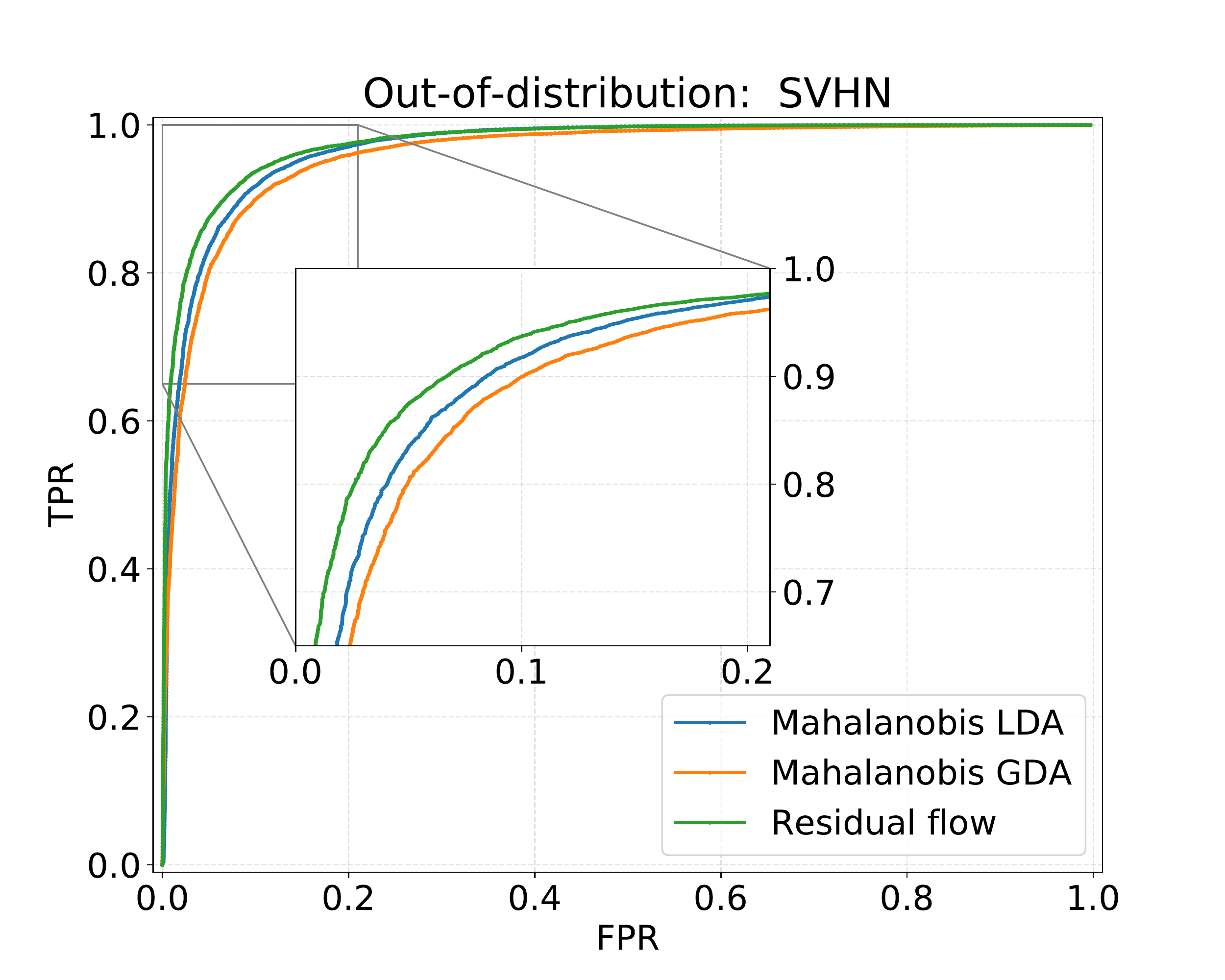} \label{fig:Resnet_CIFAR100_SVHN}}
\,
\subfigure[]
{
\includegraphics[width=0.31\textwidth]{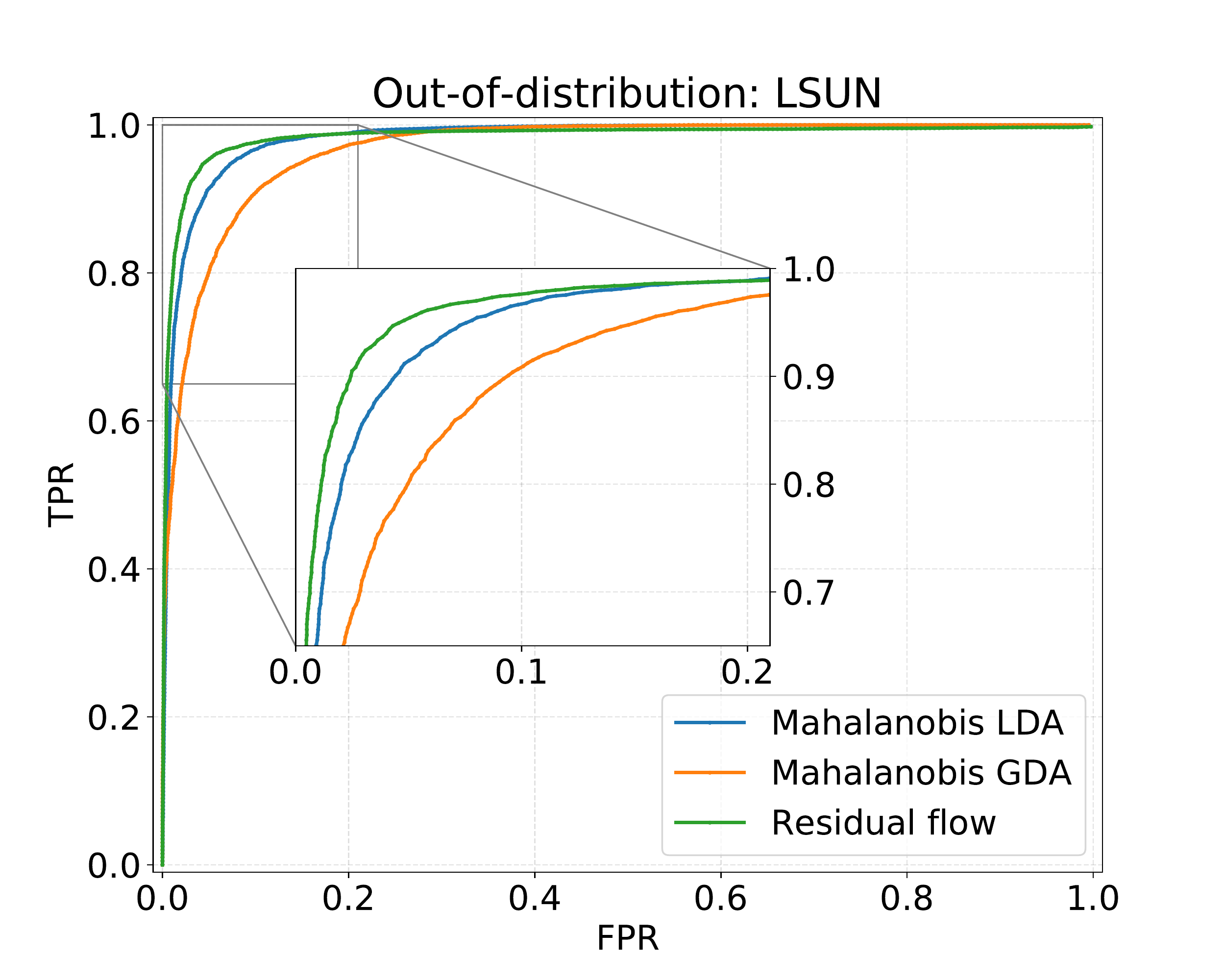} \label{fig:Resnet_CIFAR100_LSUN}}
\caption{
 Receiver operating characteristic (ROC) curve comparison of our method, Mahalanobis (LDA) and GDA for the task of OOD detection. The target network is ResNet trained on CIFAR-100. We compare the three models using the following out-of-distribution datasets: (a) TinyImageNet, (b) SVHN and (c) LSUN. 
 The x-axis and y-axis of the figures represent the false positive rate (FPR) and true positive rate (TPR), respectively.}
\label{fig:DenseNet_roc}
\end{figure*}

\begin{figure*} [ht] \centering
\subfigure[]
{
\includegraphics[width=0.31\textwidth]{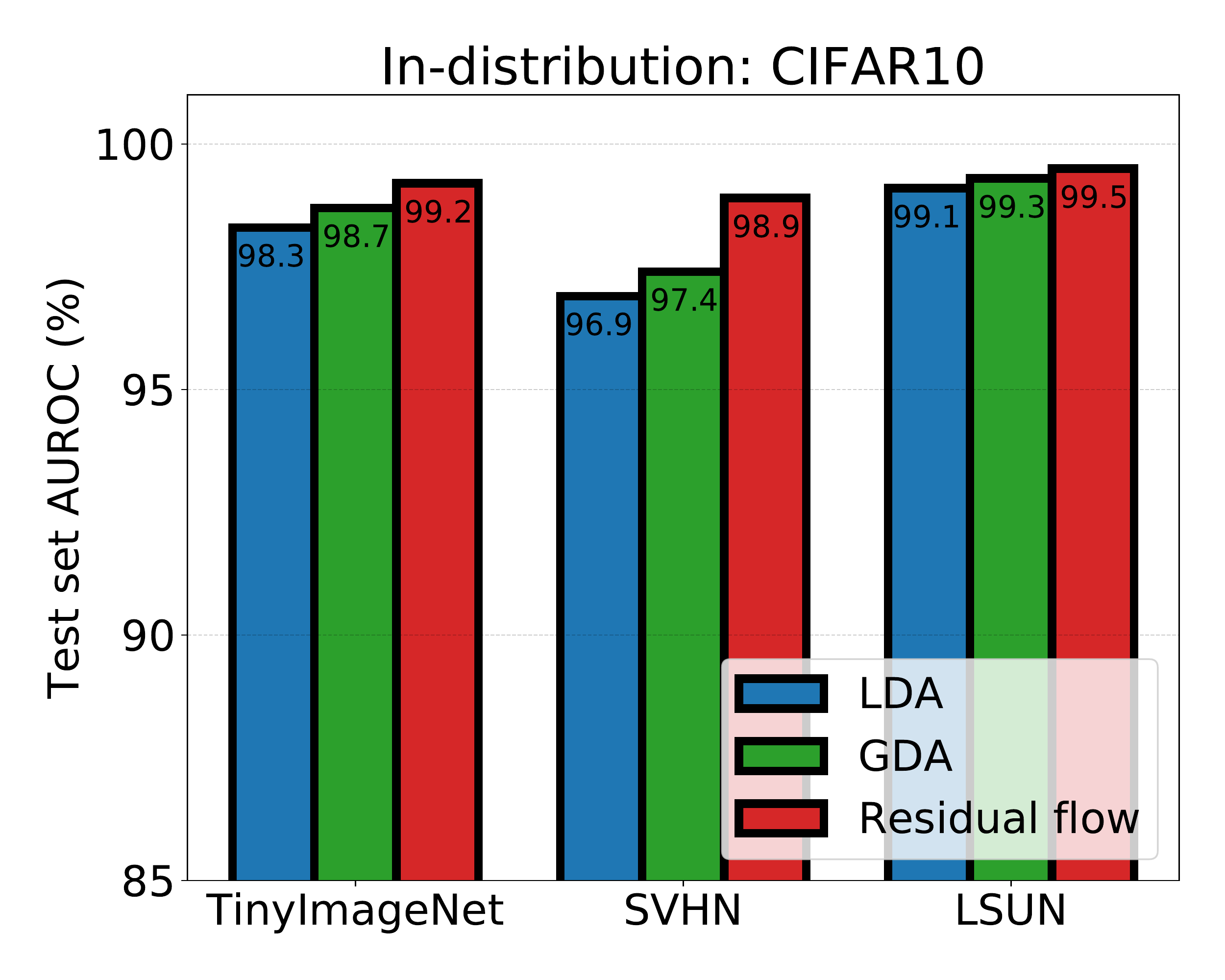} \label{fig:AUROC_Densenet_CIFAR10}}
\,
\subfigure[]
{
\includegraphics[width=0.31\textwidth]{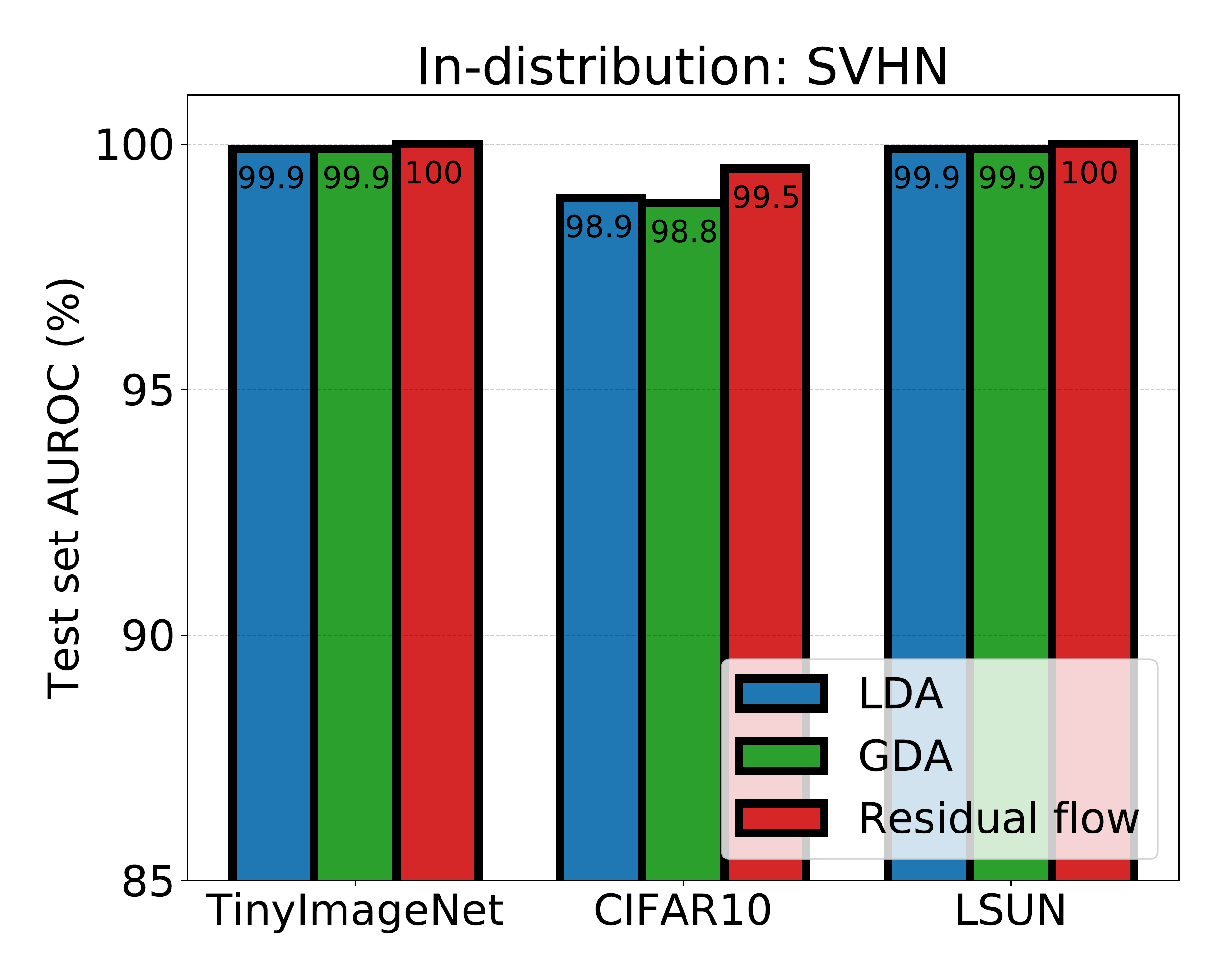} \label{fig:AUROC_Densenet_SVHN}}
\,
\subfigure[]
{
\includegraphics[width=0.31\textwidth]{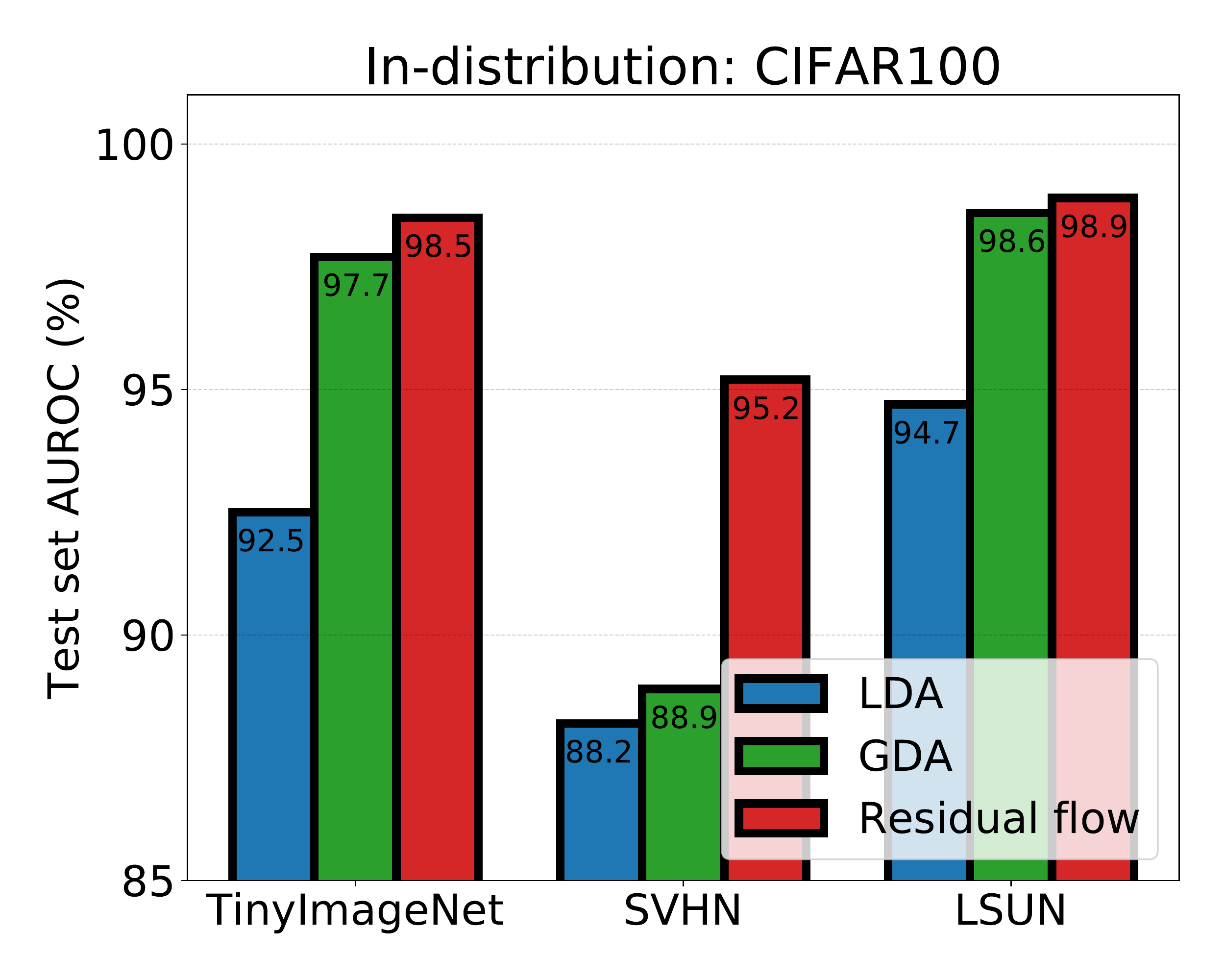} \label{fig:AUROC_Densenet_CIFAR100}}
\caption{
Area under the receiver operating characteristic (AUROC) (\%) curve comparison using DenseNet with 100 layers as a target network. We compare our results with LDA and GDA models across different in- and out-of-distribution datasets. The in-distribution datasets are: (a) CIFAR-10, (b) SVHN and (c) CIFAR-100, and the OOD datasets are presented on the x-axis of the figures. 
}
\label{fig:AUROC_Densenet}
\end{figure*}

\begin{figure*} [ht] \centering
\subfigure[]
{
\includegraphics[width=0.31\textwidth]{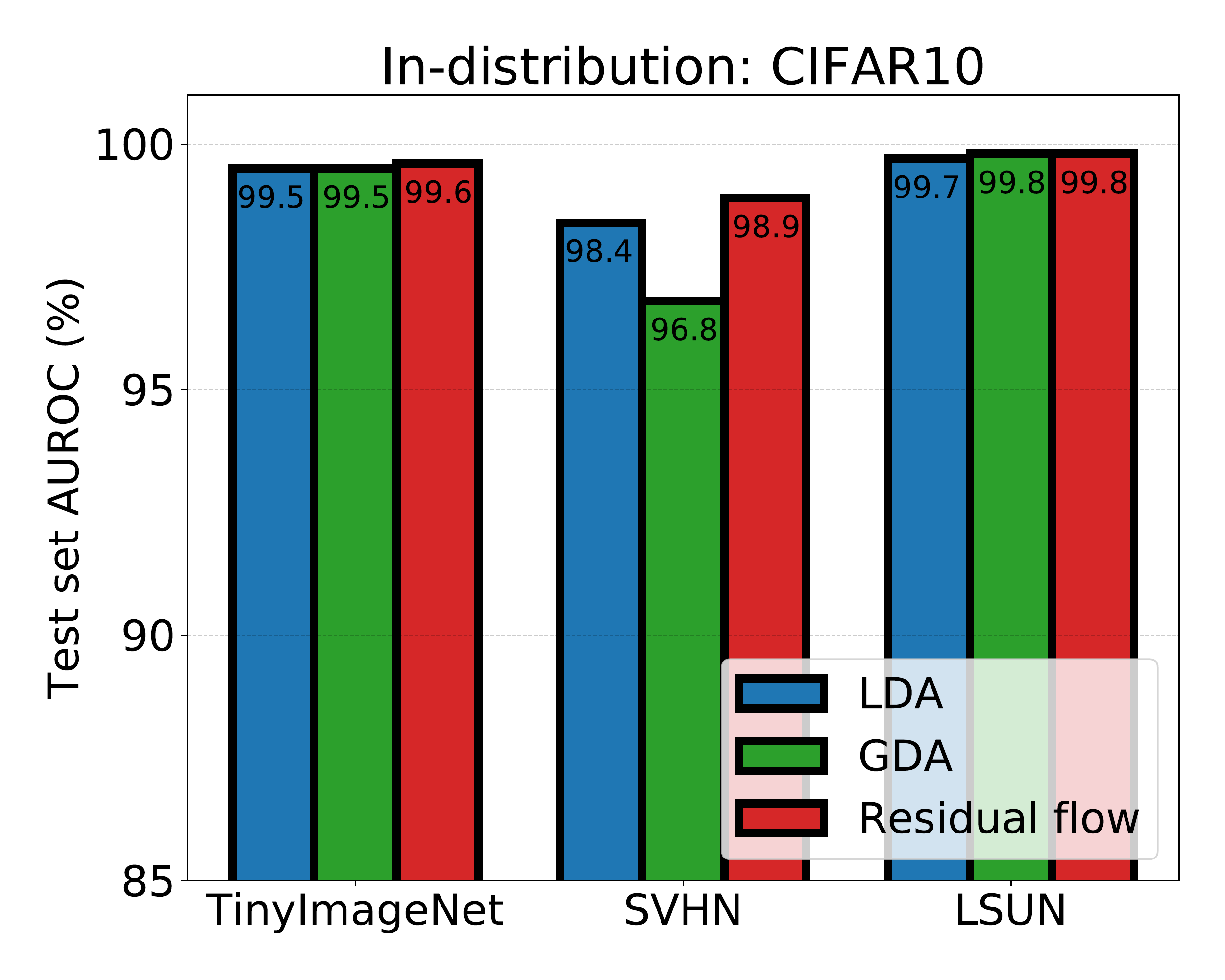}\label{fig:AUROC_Resnet_CIFAR10}}
\,
\subfigure[]
{
\includegraphics[width=0.31\textwidth]{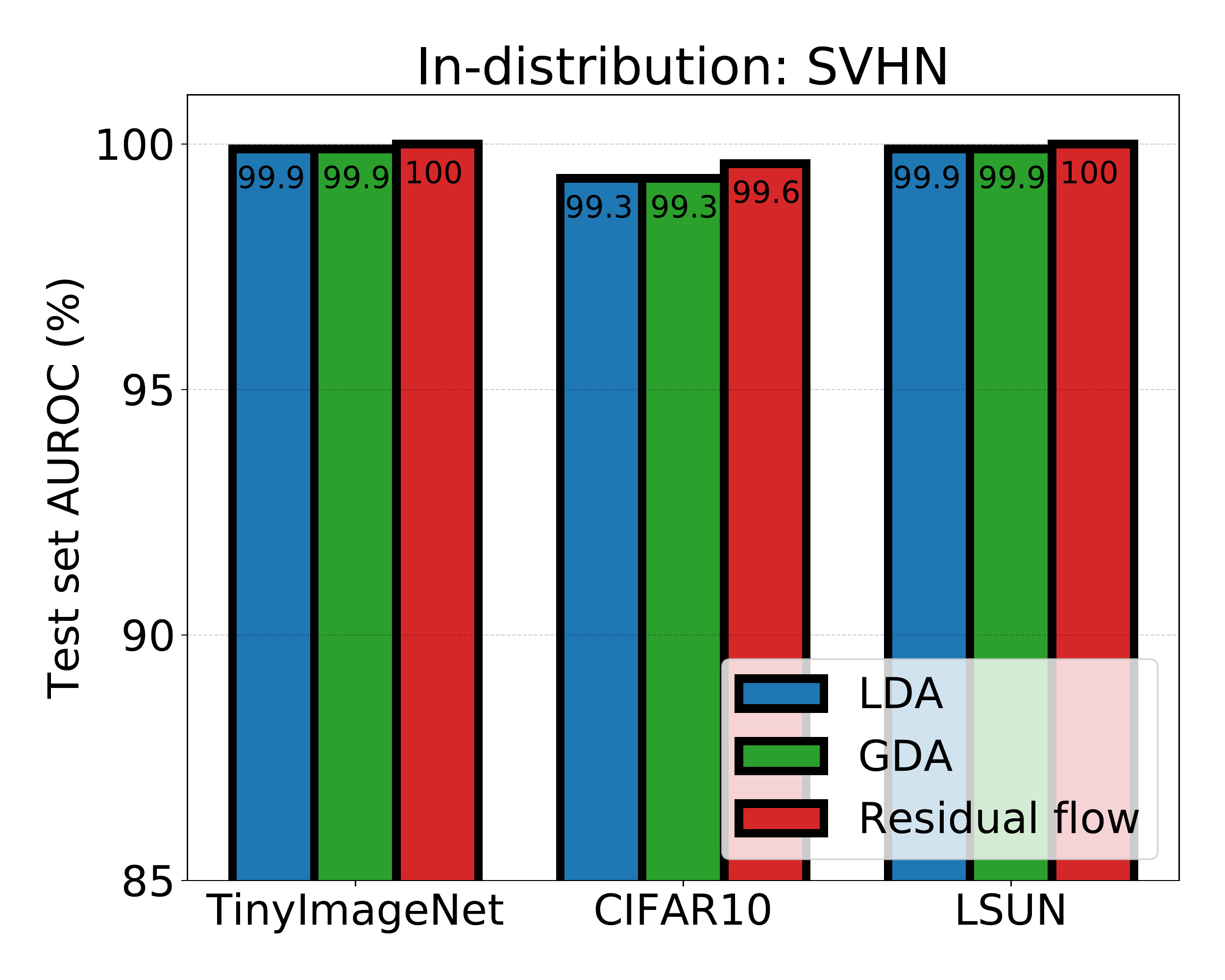} \label{fig:AUROC_Resnet_SVHN}}
\,
\subfigure[]
{
\includegraphics[width=0.31\textwidth]{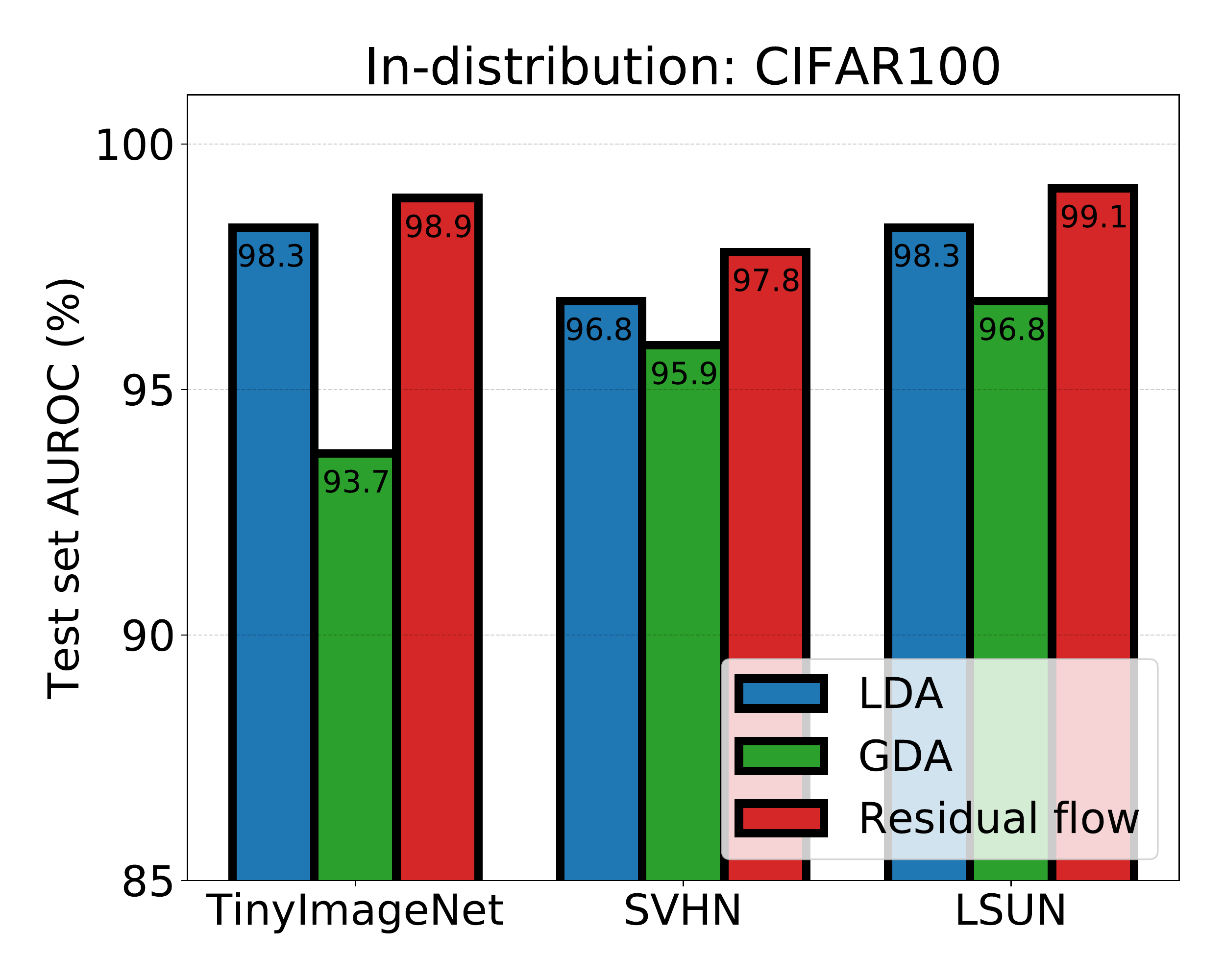} \label{fig:AUROC_Resnet_CIFAR100}}
\caption{
Area under the receiver operating characteristic (AUROC) (\%) curve comparison using ResNet with 34 layers as a target network. We compare our results with LDA and GDA models across different in- and out-of-distribution datasets. The in-distribution datasets are: (a) CIFAR-10, (b) SVHN and (c) CIFAR-100, and the OOD datasets are presented on the x-axis of the figures. }
\label{fig:AUROC_ResNet}
\end{figure*}

\section{Alternative Architecture}\label{Sec:diff_arch}
Composing a non-linear flow with linear flow blocks can be done in multiple ways. In this section, we describe an alternative residual flow architecture to the one presented in the main paper, and show that it obtains similar performance. The architecture comprises residual blocks, each composed of a single linear and several non-linear blocks. This architecture is more involved compared to the architecture in the main text, which comprises one linear flow block. We start by defining a linear flow block $f^{lin}_i$:
\begin{align}\label{eq:lin_flow}
x_1 = z_1,~~~ x_2 = x_2 \circ \exp(s_i) + t_i^T x_1,
\end{align}
where $s_i\in \mathbb{R}^{d/2}$, $t_i\in \mathbb{R}^{d/2 \times d/2}$, and $\circ$ denotes element-wise multiplication. Here $s_i$ and $t_i$ are scale and translation parameters. The scale parameters are crucial here, as without them, the Jacobian determinant is a constant $1$ by definition \cite{Dinh2014NICENI}, making the transformation volume preserving, and limiting the expressivity of the model. 
Next, we compose a residual flow block $f^{res}_i$:
\begin{equation*}
    f^{res}_i = f^{lin}_i \cdot p_i \cdot f^{non-lin}_{i,1} \cdot r \cdot f^{non-lin}_{i,2} \cdot r \cdot p_i^{-1},
\end{equation*} 
where the linear flow block $f^{lin}_i$ was defined above, $r$ is a switch permutation, $p_i$ is a permutation matrix and $p_i^{-1}$ is its inverse, and $f^{non-lin}_{i,1},f^{non-lin}_{i,2}$ are non-linear blocks as described in Eq. \eqref{eq:3} in the main paper. We then compose a residual flow model as:
\begin{equation*}
    f_{res} = f^{res}_1\cdot r \cdot f^{res}_2 \dots r \cdot f^{res}_k.
\end{equation*} 
Note that, from Eq. \eqref{eq:3} in the main paper, when $s_i(\cdot)=0$ and $t_i(\cdot)=0$, the non-linear terms $f^{non-lin}_i$ are just the identity, the permutation terms cancel each other, and in that case the residual flow $f^{res}$ is equivalent to the linear flow $f^{lin}$. Thus, we pre-train the residual flow by fixing the networks $s_i(\cdot)$ and $t_i(\cdot)$ to be zero, which is equivalent to fitting a Gaussian distribution model to our data\footnote{The stopping condition for this stage is when the Kullback–Leibler divergence measure between the linear flow $\hat{p}_X$ and the Gaussian distribution calculated using the empirical covariance $\Tilde{p}_X$ meets the criteria: $\mathcal{D}_{KL}\left(\hat{p}_X||\Tilde{p}_X\right)< 10^{-4}$.}. In practice, setting only the last layer of the networks for $s_i(\cdot)$ and $t_i(\cdot)$ to zero is enough, and we found this to perform better in fine tuning the non-linear terms, as most of the network is not initialized to zero. Then, we fine tune the non-linear components of the model to obtain a better fit to the data. Figures \ref{figure:alternative_arc} illustrates the alternative architecture. This architecture achieves similar results to that proposed in the main paper (see Tables \ref{tbl:In_diff_architecture} and \ref{tbl:out_diff_arvhitecture} for full comparison), but with the extra time overhead of training the linear flow. Hence, we chose to include the simpler architecture in the main paper. 

\begin{table*}[ht] 
\centering
\resizebox{\textwidth}{!}{
\begin{tabular}{@{}ccclclclll@{}}
\toprule
\multirow{2}{*}{\begin{tabular}[c]{@{}c@{}} In-dist \\ (model) \end{tabular}}
 & \multirow{2}{*}{\begin{tabular}[c]{@{}c@{}} Out-of-dist\end{tabular}}
& \multicolumn{1}{c}{\begin{tabular}[c]{@{}c@{}} TNR at TPR 95\% \end{tabular}}
& \multicolumn{1}{c}{\begin{tabular}[c]{@{}c@{}} AUROC \end{tabular}}
& \multicolumn{1}{c}{\begin{tabular}[c]{@{}c@{}} Detection accuracy \end{tabular}} 
& \multicolumn{1}{c}{\begin{tabular}[c]{@{}c@{}} AUPR in \end{tabular}} 
& \multicolumn{1}{c}{\begin{tabular}[c]{@{}c@{}} AUPR out \end{tabular}} \\\cline{3-7}  
\multicolumn{1}{c}{} & \multicolumn{1}{c}{} & \multicolumn{5}{c}{  Mahalanobis \cite{lee2018simple}/ Res-Flow without pre-processing / Res-Flow with pre-processing} \\ \midrule
\multirow{3}{*}{\begin{tabular}[c]{@{}c@{}} CIFAR-10 \\(DenseNet) \end{tabular}} 
& \multicolumn{1}{c}{SVHN} 
& \multicolumn{1}{c}{88.4 / 92.7 / {\bf 94.4}}
& \multicolumn{1}{c}{96.8 / 98.5 / {\bf 98.8}}
& \multicolumn{1}{c}{92.4 / 94.0 / {\bf 94.8}}
& \multicolumn{1}{c}{98.7 / 99.4 / {\bf 99.5}}
& \multicolumn{1}{c}{90.4 / 96.6 / {\bf 97.6}}\\
& \multicolumn{1}{c}{ImageNet}
& \multicolumn{1}{c}{95.4 / {\bf 97.3} / {\bf 97.3}}
& \multicolumn{1}{c}{98.8 / {\bf 99.3} / {\bf 99.3}}
& \multicolumn{1}{c}{95.3 / {\bf 96.3} / {\bf 96.3}}
& \multicolumn{1}{c}{98.9 / {\bf 99.3} / {\bf 99.3}}
& \multicolumn{1}{c}{98.7 / {\bf 99.3} / {\bf 99.3}}\\
& \multicolumn{1}{c}{LSUN}  
& \multicolumn{1}{c}{97.3 / {\bf 98.4} / {\bf 98.4}}
& \multicolumn{1}{c}{99.0 / {\bf 99.6} / {\bf 99.6}}
& \multicolumn{1}{c}{96.2 / {\bf 97.4} / {\bf 97.4}}
& \multicolumn{1}{c}{99.1 / {\bf 99.5} / {\bf 99.5}}
& \multicolumn{1}{c}{98.8 / {\bf 99.6} / {\bf 99.6}}\\ \midrule
\multirow{3}{*}{\begin{tabular}[c]{@{}c@{}} CIFAR-100 \\(DenseNet) \end{tabular}} 
& \multicolumn{1}{c}{SVHN}  
& \multicolumn{1}{c}{84.1 / 68.0 / {\bf 87.1}}
& \multicolumn{1}{c}{96.2 / 92.8 / {\bf 96.8}}
& \multicolumn{1}{c}{91.0 / 85.3 / {\bf 91.1}}
& \multicolumn{1}{c}{98.6 / 96.6 / {\bf 98.6}}
& \multicolumn{1}{c}{89.2 / 85.5 / {\bf 94.4}}\\
& \multicolumn{1}{c}{TinyImageNet} 
& \multicolumn{1}{c}{77.5 / 93.1 / \bf {93.4}}
& \multicolumn{1}{c}{95.4 / {\bf 98.5} / \bf {98.5}}
& \multicolumn{1}{c}{89.2 / 94.1 / \bf {94.3}}
& \multicolumn{1}{c}{95.8 / {\bf 98.4} / \bf {98.4}}
& \multicolumn{1}{c}{93.8 / {\bf 98.5} / \bf {98.5}}\\
& \multicolumn{1}{c}{LSUN} 
& \multicolumn{1}{c}{69.4 / {\bf 95.3} / {\bf 95.3}}
& \multicolumn{1}{c}{94.6 / {\bf 98.8} / {\bf 98.8}}
& \multicolumn{1}{c}{89.2 / {\bf 95.4} / {\bf 95.4}}
& \multicolumn{1}{c}{95.3 / {\bf 98.5} / {\bf 98.5}}
& \multicolumn{1}{c}{92.7/ {\bf 98.9} / {\bf 98.9}}\\ \midrule
\multirow{3}{*}{\begin{tabular}[c]{@{}c@{}} SVHN \\(DenseNet) \end{tabular}} 
& \multicolumn{1}{c}{CIFAR-10}  
& \multicolumn{1}{c}{95.8 / 96.9 / {\bf 97.5}}
& \multicolumn{1}{c}{98.8 / 99.2 / {\bf 99.3}}
& \multicolumn{1}{c}{95.8 / 96.7 / {\bf 97.0}}
& \multicolumn{1}{c}{95.4 / 96.9 / {\bf 97.4}}
& \multicolumn{1}{c}{99.6 / 99.7 / {\bf 99.8}}\\
& \multicolumn{1}{c}{TinyImageNet} 
& \multicolumn{1}{c}{99.6 / {\bf 99.8} / {\bf 99.8}}
& \multicolumn{1}{c}{99.9 / {\bf 99.9} / {\bf 99.9}}
& \multicolumn{1}{c}{98.9 / {\bf 99.2} / {\bf 99.2}}
& \multicolumn{1}{c}{99.6 / {\bf 99.8} / {\bf 99.8}}
& \multicolumn{1}{c}{100.0 / {\bf 100.0} / {\bf 100.0}}\\
& \multicolumn{1}{c}{LSUN}  
& \multicolumn{1}{c}{99.7 / {\bf 99.8} / {\bf 99.8}}
& \multicolumn{1}{c}{99.9 / {\bf 100.0} / {\bf 100.0}}
& \multicolumn{1}{c}{99.3 / {\bf 99.5} / {\bf 99.5}}
& \multicolumn{1}{c}{99.7 / {\bf 99.9} / {\bf 99.9}}
& \multicolumn{1}{c}{100.0 / {\bf 100.0} / {\bf 100.0}}\\ \midrule
\multirow{3}{*}{\begin{tabular}[c]{@{}c@{}} CIFAR-10 \\(ResNet) \end{tabular}} 
& \multicolumn{1}{c}{SVHN} 
& \multicolumn{1}{c}{96.2 / 91.7  / {\bf 96.5}}
& \multicolumn{1}{c}{99.1 / 98.3 / {\bf 99.2}}
& \multicolumn{1}{c}{95.8  / 93.5 / {\bf 95.9}}
& \multicolumn{1}{c}{99.6 / 99.3 / {\bf 99.7}}
& \multicolumn{1}{c}{98.3/ 96.4 / {\bf 98.3}}\\
& \multicolumn{1}{c}{TinyImageNet}  
& \multicolumn{1}{c}{97.4 / {\bf 98.9} / 98.3}
& \multicolumn{1}{c}{99.5 / {\bf 99.8} / 99.6}
& \multicolumn{1}{c}{96.3 / {\bf 97.6} / 97.1}
& \multicolumn{1}{c}{99.5 / {\bf 99.7} / 99.6}
& \multicolumn{1}{c}{99.5 / {\bf 99.7} / 99.6}\\
& \multicolumn{1}{c}{LSUN}
& \multicolumn{1}{c}{98.7 / {\bf 99.3} / 99.1}
& \multicolumn{1}{c}{99.7 / {\bf 99.8} /{\bf  99.8}}
& \multicolumn{1}{c}{97.5 / 97.8 / {\bf 97.9}}
& \multicolumn{1}{c}{99.7 / {\bf 99.8} / {\bf 99.8}}
& \multicolumn{1}{c}{99.7 / {\bf 99.8} / {\bf 99.8}}\\ \midrule
\multirow{3}{*}{\begin{tabular}[c]{@{}c@{}} CIFAR-100 \\(ResNet) \end{tabular}} 
& \multicolumn{1}{c}{SVHN}   
& \multicolumn{1}{c}{92.4 / 83.4 / {\bf 94.0}}
& \multicolumn{1}{c}{98.2 / 96.5 / {\bf 98.5}}
& \multicolumn{1}{c}{93.8 / 90.3 / {\bf 94.6}}
& \multicolumn{1}{c}{99.2 / 98.6 / {\bf 99.3}}
& \multicolumn{1}{c}{96.2 / 92.7 / {\bf 97.2}}\\
& \multicolumn{1}{c}{TinyImageNet} 
& \multicolumn{1}{c}{89.4 / {\bf 95.0} / {\bf 95.0}}
& \multicolumn{1}{c}{97.9 / 98.9 / {\bf 99.9}}
& \multicolumn{1}{c}{92.7 / {\bf 95.0} / {\bf 95.0}}
& \multicolumn{1}{c}{97.9 / {\bf 98.9} / {\bf 98.9}}
& \multicolumn{1}{c}{97.9 / {\bf 98.8} / {\bf 98.8}}\\
& \multicolumn{1}{c}{LSUN}
& \multicolumn{1}{c}{92.8 / {\bf 96.2} / {\bf 96.2}}
& \multicolumn{1}{c}{98.3 / {\bf 99.2} / 99.1}
& \multicolumn{1}{c}{93.9 / {\bf 95.6} / {\bf 95.6}}
& \multicolumn{1}{c}{97.9 / {\bf 99.0} / {\bf 99.0}}
& \multicolumn{1}{c}{98.5 / {\bf 99.2} / {\bf 99.2}}\\ \midrule
\multirow{3}{*}{\begin{tabular}[c]{@{}c@{}} SVHN \\(ResNet) \end{tabular}} 
& \multicolumn{1}{c}{CIFAR-10} 
& \multicolumn{1}{c}{97.6 / {\bf 98.6} / 98.5}
& \multicolumn{1}{c}{99.3 / {\bf 99.6} / {\bf 99.6}}
& \multicolumn{1}{c}{96.9 / {\bf 97.8} / 97.7}
& \multicolumn{1}{c}{97.3 / {\bf 98.2} / 98.1}
& \multicolumn{1}{c}{99.7 / {\bf 99.9} / {\bf 99.9}}\\
& \multicolumn{1}{c}{TinyImageNet} 
& \multicolumn{1}{c}{99.7 / {\bf 99.8} / {\bf 99.8}}
& \multicolumn{1}{c}{99.8 / {\bf 99.9} / {\bf 99.9}}
& \multicolumn{1}{c}{99.1 / {\bf 99.4} / {\bf 99.4}}
& \multicolumn{1}{c}{99.5 / {\bf 99.7} / {\bf 99.7}}
& \multicolumn{1}{c}{99.9 / {\bf 100.0} / {\bf 100.0}}\\
& \multicolumn{1}{c}{LSUN} 
& \multicolumn{1}{c}{99.8 / {\bf 99.9} / {\bf 99.9}}
& \multicolumn{1}{c}{99.9 / {\bf 100.0} / {\bf 100.0}}
& \multicolumn{1}{c}{99.6 / {\bf 99.7} / {\bf 99.7}}
& \multicolumn{1}{c}{99.6 / {\bf 99.7} / {\bf 99.7}}
& \multicolumn{1}{c}{99.9 / {\bf 100.0} / {\bf 100.0}}\\ 
\bottomrule
\end{tabular}}
\vspace{+0.02in}
\caption{A comparison between residual flow implemented using the architecture described in Section \ref{Sec:diff_arch} and Mahalanobis \cite{lee2018simple} on the task of out-of-distribution detection for image classification of various in- and out-of-distribution data sets. The hyper-parameters were tuned using a validation set of in- and out-of-distribution datasets. The values presented here are percentages and the best results are indicated in bold.}
\label{tbl:In_diff_architecture}
\end{table*}

\begin{table*}[ht] 
\centering
\resizebox{\textwidth}{!}{
\begin{tabular}{@{}ccclclclll@{}}
\toprule
\multirow{2}{*}{\begin{tabular}[c]{@{}c@{}} In-dist \\ (model) \end{tabular}}
 & \multirow{2}{*}{\begin{tabular}[c]{@{}c@{}} Out-of-dist\end{tabular}}
& \multicolumn{1}{c}{\begin{tabular}[c]{@{}c@{}} TNR at TPR 95\% \end{tabular}}
& \multicolumn{1}{c}{\begin{tabular}[c]{@{}c@{}} AUROC \end{tabular}}
& \multicolumn{1}{c}{\begin{tabular}[c]{@{}c@{}} Detection accuracy \end{tabular}} 
& \multicolumn{1}{c}{\begin{tabular}[c]{@{}c@{}} AUPR in \end{tabular}} 
& \multicolumn{1}{c}{\begin{tabular}[c]{@{}c@{}} AUPR out \end{tabular}} \\\cline{3-7}  
\multicolumn{1}{c}{} & \multicolumn{1}{c}{} & \multicolumn{5}{c}{ Mahalanobis \cite{lee2018simple}/ Res-Flow without pre-processing / Res-Flow with pre-processing} \\ \midrule
\multirow{3}{*}{\begin{tabular}[c]{@{}c@{}} CIFAR-10 \\(DenseNet) \end{tabular}} 
& \multicolumn{1}{c}{SVHN} 
& \multicolumn{1}{c}{89.6 / 75.6 / {\bf 91.7} }
& \multicolumn{1}{c}{97.6 / 94.9 / {\bf 98.0} }
& \multicolumn{1}{c}{92.6 / 87.8 / {\bf 93.4} }
& \multicolumn{1}{c}{94.5 / 88.7 / {\bf 96.2} }
& \multicolumn{1}{c}{99.0 / 97.9 / {\bf 99.1} }\\
& \multicolumn{1}{c}{TinyImageNet}  
& \multicolumn{1}{c}{94.9 / {\bf 97.3} / {\bf 97.3} }
& \multicolumn{1}{c}{98.8 / {\bf 99.3} / {\bf 99.3} }
& \multicolumn{1}{c}{95.0 / {\bf 96.4} / {\bf 96.4} }
& \multicolumn{1}{c}{98.7 / {\bf 99.4} / {\bf 99.4} }
& \multicolumn{1}{c}{98.8 / {\bf 99.3} / {\bf 99.3} }\\
& \multicolumn{1}{c}{LSUN}  
& \multicolumn{1}{c}{97.2 / {\bf 98.4} / {\bf 98.4} }
& \multicolumn{1}{c}{99.2 / {\bf 99.6} / {\bf 99.6} }
& \multicolumn{1}{c}{96.2 / {\bf 97.4} / {\bf 97.4} }
& \multicolumn{1}{c}{99.3 / {\bf 99.6} / {\bf 99.6} }
& \multicolumn{1}{c}{99.2 / {\bf 99.6} / {\bf 99.6} }\\ \midrule
\multirow{3}{*}{\begin{tabular}[c]{@{}c@{}} CIFAR-100 \\(DenseNet) \end{tabular}} 
& \multicolumn{1}{c}{SVHN}
& \multicolumn{1}{c}{ 62.2 / 65.4 / {\bf 86.3} }
& \multicolumn{1}{c}{91.8 / 91.7 / {\bf 96.4} }
& \multicolumn{1}{c}{84.6 / 84.2 / {\bf 90.7} }
& \multicolumn{1}{c}{82.6 / 83.9 / {\bf 94.0}}
& \multicolumn{1}{c}{95.8 / 96.0 / {\bf 98.3}}\\
& \multicolumn{1}{c}{TinyImageNet} 
& \multicolumn{1}{c}{87.2 / {\bf 92.4} / 91.2 }
& \multicolumn{1}{c}{97.0 / {\bf 98.3} / 98.1 }
& \multicolumn{1}{c}{91.8 / {\bf 93.7} / 93.4 }
& \multicolumn{1}{c}{96.2 / {\bf 98.2} / 98.1 }
& \multicolumn{1}{c}{97.1 / {\bf 98.3} / 98.2 }\\
& \multicolumn{1}{c}{LSUN} 
& \multicolumn{1}{c}{91.4 / 95.1 / {\bf 95.3 }}
& \multicolumn{1}{c}{97.9 / 98.7 / {\bf 98.8 }}
& \multicolumn{1}{c}{93.8 / 95.1 / {\bf 95.3 }}
& \multicolumn{1}{c}{98.1 / 98.5 / {\bf 98.6 }}
& \multicolumn{1}{c}{97.6 / {\bf 98.9} / {\bf 98.9 }}\\ \midrule
\multirow{3}{*}{\begin{tabular}[c]{@{}c@{}} SVHN \\(DenseNet) \end{tabular}} 
& \multicolumn{1}{c}{CIFAR-10}  
& \multicolumn{1}{c}{97.5 / 96.2 / {\bf 96.5} }
& \multicolumn{1}{c}{98.8 / 98.9 / {\bf 99.1}}
& \multicolumn{1}{c}{96.3 / 96.1 / {\bf 96.3}}
& \multicolumn{1}{c}{99.6 /{\bf 99.7} / {\bf 99.7}}
& \multicolumn{1}{c}{95.1 / 96.0 / {\bf 96.5}}\\
& \multicolumn{1}{c}{TinyImageNet}
& \multicolumn{1}{c}{99.9/ 99.7 / {\bf 99.9 }}
& \multicolumn{1}{c}{99.8 / 99.9 / {\bf 99.9}}
& \multicolumn{1}{c}{98.9 / {\bf 99.1} / 99.0}
& \multicolumn{1}{c}{99.9 / 99.8 / {\bf 99.9 }}
& \multicolumn{1}{c}{99.5 / {\bf 100.0} / 99.6}\\
& \multicolumn{1}{c}{LSUN}  
& \multicolumn{1}{c}{100.0 / 99.8 / {\bf 100.0 }}
& \multicolumn{1}{c}{99.9 / {\bf 99.9} / {\bf 99.9 }}
& \multicolumn{1}{c}{99.2 / {\bf 99.4} / 99.3}
& \multicolumn{1}{c}{99.9 / 99.8 / {\bf 100.0}}
& \multicolumn{1}{c}{99.6 / {\bf 99.9} / 99.7 }\\ \midrule
\multirow{3}{*}{\begin{tabular}[c]{@{}c@{}} CIFAR-10 \\(ResNet) \end{tabular}} 
& \multicolumn{1}{c}{SVHN}
& \multicolumn{1}{c}{75.8 /76.0 / {\bf  95.7 }}
& \multicolumn{1}{c}{95.5 / 94.2 / {\bf 98.9 }}
& \multicolumn{1}{c}{89.1 / 87.1 / {\bf 95.6}}
& \multicolumn{1}{c}{91.0 / 97.4 /  {\bf 99.4 }}
& \multicolumn{1}{c}{98.0 / 89.3 / {\bf 98.0 }}\\
& \multicolumn{1}{c}{TinyIageNet}
& \multicolumn{1}{c}{95.5 / {\bf 98.8}/ 98.5  }
& \multicolumn{1}{c}{99.0 / {\bf 99.7} / 99.6 }
& \multicolumn{1}{c}{95.4 / {\bf 97.4} / 97.1 }
& \multicolumn{1}{c}{98.6 / {\bf 99.7} / 99.6 }
& \multicolumn{1}{c}{99.1 / {\bf 99.7} / 99.6 }\\
& \multicolumn{1}{c}{LSUN} 
& \multicolumn{1}{c}{98.1 / 99.5 / {\bf 99.6 }}
& \multicolumn{1}{c}{99.5 / 99.8 / {\bf 99.9}}
& \multicolumn{1}{c}{97.2 / 98.2 / {\bf 98.5 }}
& \multicolumn{1}{c}{99.5 /{\bf  99.8} / {\bf 99.8 }}
& \multicolumn{1}{c}{99.5 / 99.8 / {\bf 99.9 }}\\ \midrule
\multirow{3}{*}{\begin{tabular}[c]{@{}c@{}} CIFAR-100 \\(ResNet) \end{tabular}}
& \multicolumn{1}{c}{SVHN}
& \multicolumn{1}{c}{41.9 / 59.1 / {\bf 66.8 }}
& \multicolumn{1}{c}{84.4 / 90.6 / {\bf 92.4 }}
& \multicolumn{1}{c}{76.5 / 82.6 / {\bf 84.9 }}
& \multicolumn{1}{c}{69.1 / 81.0 / {\bf 83.3 }}
& \multicolumn{1}{c}{92.7 / 95.8 / {\bf 96.8 }}\\
& \multicolumn{1}{c}{TinyImageNet}    
& \multicolumn{1}{c}{70.3 / 73.9 / {\bf 77.3} }
& \multicolumn{1}{c}{87.9 / 88.8 / {\bf 89.6} }
& \multicolumn{1}{c}{84.6 / 84.5 / {\bf 86.7} }
& \multicolumn{1}{c}{76.8 / 78.8 / {\bf 79.2} }
& \multicolumn{1}{c}{90.7 / 88.8 / {\bf 92.5} }\\
& \multicolumn{1}{c}{LSUN}   
& \multicolumn{1}{c}{56.6 / 66.1 / {\bf 68.1} }
& \multicolumn{1}{c}{82.3 / {\bf 89.1} / 86.5 }
& \multicolumn{1}{c}{79.7 / {\bf 85.6} / 83.4 }
& \multicolumn{1}{c}{70.3 / {\bf 79.1} / 75.8 }
& \multicolumn{1}{c}{85.3 / 89.2 / {\bf 89.7} }\\ \midrule
\multirow{3}{*}{\begin{tabular}[c]{@{}c@{}} SVHN \\(ResNet) \end{tabular}} 
& \multicolumn{1}{c}{CIFAR-10} 
& \multicolumn{1}{c}{94.1 / {\bf 98.4} / 97.6 }
& \multicolumn{1}{c}{97.6 / {\bf 99.5} / 99.2 }
& \multicolumn{1}{c}{94.6 / {\bf 97.5} / 96.4 }
& \multicolumn{1}{c}{98.1 / {\bf 99.9} / 99.7 }
& \multicolumn{1}{c}{94.7 / {\bf 97.9} / 97.3 }\\
& \multicolumn{1}{c}{TinyImageNet} 
& \multicolumn{1}{c}{99.2 / {\bf  99.9} / {\bf 99.9}}
& \multicolumn{1}{c}{99.3 / {\bf 99.9} / {\bf 99.9 }}
& \multicolumn{1}{c}{98.8 / {\bf 99.5} / {\bf 99.5} }
& \multicolumn{1}{c}{98.8 / {\bf 99.7} / {\bf 99.9}}
& \multicolumn{1}{c}{98.3 / {\bf 100.0} / 99.6 }\\
& \multicolumn{1}{c}{LSUN}
& \multicolumn{1}{c}{99.9 / 99.9 / {\bf  100.0 }}
& \multicolumn{1}{c}{99.9 / {\bf 100.0} / 99.9 }
& \multicolumn{1}{c}{99.5 / {\bf 99.7} / 99.6 }
& \multicolumn{1}{c}{99.9 / 99.7 / {\bf 99.9 }}
& \multicolumn{1}{c}{98.8 / {\bf 100.0} / {\bf 100.0 }}\\ 
\bottomrule
\end{tabular}}
\vspace{+0.02in}
\caption{A comparison between residual flow implemented using the architecture described in Section \ref{Sec:diff_arch} and Mahalanobis \cite{lee2018simple} on the task of out-of-distribution detection for image classification of various in- and out-of-distribution data sets. The hyper-parameters were tuned using strictly in-distribution and adversarial (FGSM) samples. The values presented here are percentages and the best results are indicated in bold.}
\label{tbl:out_diff_arvhitecture}
\end{table*}

\end{document}